\documentclass[twoside,11pt]{article}

\usepackage[T1]{fontenc}
\usepackage[utf8]{inputenc}
\usepackage{color}
\usepackage[english]{babel}
\usepackage{bm}

\usepackage{amsmath, amsthm, amssymb}

 \RequirePackage{hyperref}
 \hypersetup{%
            unicode=true,
            bookmarksnumbered=false,
            bookmarksopen=false,
            breaklinks=false,
            pdfborder={0 0 0},
            pdfborderstyle={},
            colorlinks=true
 }

\makeatletter
\theoremstyle{plain}
\newtheorem{asm}{\protect\assumptionname}
\theoremstyle{plain}

\theoremstyle{plain}
\newtheorem{thm}{\protect\theoremname}
\theoremstyle{plain}
\newtheorem{lem}{\protect\lemmaname}
\theoremstyle{plain}

\theoremstyle{plain}

\theoremstyle{definition}
\newtheorem*{prob*}{\protect\problemname}

\usepackage[dvipsnames]{xcolor}
\RequirePackage[round]{natbib}
\usepackage{geometry} 
\usepackage[bb=boondox]{mathalfa}
\usepackage{booktabs}       
\usepackage{nicefrac}       
\usepackage{microtype,lmodern}
\RequirePackage{subcaption}

\newcommand*{\at}{@}
\makeatother

\providecommand{\assumptionname}{Assumption}
\providecommand{\factname}{Fact}
\providecommand{\theoremname}{Theorem}
\providecommand{\lemmaname}{Lemma}
\providecommand{\corollaryname}{Corollary}
\providecommand{\propositionname}{Proposition}
\providecommand{\problemname}{Problem}

\input{./macros.tex}

\title{Convex Programming for Estimation\\ in Nonlinear Recurrent Models}

\author{Sohail Bahmani\\ 
        School of Electrical \& Computer Engineering\\
         Georgia Institute of Technology\\
         email: \texttt{sohail.bahmani\at ece.gatech.edu}
		\and
		Justin Romberg\\
        School of Electrical \& Computer Engineering\\
         Georgia Institute of Technology\\
         email: \texttt{jrom\at ece.gatech.edu}
         }

\begin{document}

\maketitle

\begin{abstract}
    We propose a formulation for nonlinear recurrent models that includes simple parametric models of recurrent neural networks as a special case. The proposed formulation leads to a natural estimator in the form of a convex program. We provide a sample complexity for this estimator in the case of stable dynamics, where the nonlinear recursion has a certain contraction property, and under certain regularity conditions on the input distribution. We evaluate the performance of the estimator by simulation on synthetic data. These numerical experiments also suggest the extent at which the imposed theoretical assumptions may be relaxed.\looseness=-1
\end{abstract}


\section{Introduction}
Given a \emph{differentiable} and \emph{convex} function $f:\mbb R^{n}\to\mbb R$ with
$\nabla f\left(\mb 0\right)=\mb 0$, we consider the dynamics described
by the recursion
\begin{align}
\mb x_{t} & =\nabla f\left(\tg{\mb A}\mb x_{t-1}+\tg{\mb B}\mb u_{t-1}\right)\,, \label{eq:nonlinear-dynamics}
\end{align}
 where $\mb u_{0},\mb u_{1},\dotsc$ are i.i.d. copies of a random
vector $\mb u\in\mbb R^{p}$, the initial state $\mb x_{0}$ is zero,
and the matrices $\tg{\mb A}\in\mbb R^{n\times n}$ and $\tg{\mb B}\in\mbb R^{n\times p}$
are the parameters of the model. In the setup described above, we want to address the following problem.
\begin{prob*}
    Given a time horizon $T$, estimate the model parameters $\tg{\mb A}$ and $\tg{\mb B}$, from a single observed trajectory $ (\mb u_0,\mb x_0=\mb0), (\mb u_1, \mb x_1), \dotsc,(\mb u_T,\mb x_T)$.
\end{prob*}

The specific form of the nonlinearity in \eqref{eq:nonlinear-dynamics} might seem strange at first, but many common choices of nonlinearities used in practice are special cases of this formulation. For instance, increasing nonlinearities that act coordinate-wise can be modeled by choosing the appropriate \emph{separable} convex function $f$ in \eqref{eq:nonlinear-dynamics}. Particularly, the (parameterized) ReLU function $x\mapsto x_+ + c(-x)_+$ for some constant $c\ge 0$, which is popular in neural network models, corresponds to the choice of $f(\mb x)=\sum_{i=1}^n(x_i)_+^2/2+c(-x_i)_+^2/2$ in our proposed model.\looseness=-1

With  $\beta>0$ denoting a sufficiently large normalizing constant, we collect the ground truth parameters in $\tg{\mb C}=\bmx{\tg{\mb A} & \beta^{-1}\tg{\mb B}}$
and for $t=0,1,\dotsc,$ we set
\begin{align*}
\mb z_{t} & =\bmx{\mb x_{t}\\
\beta\,\mb u_{t}
}\,.
\end{align*}
Therefore, the dynamics can be equivalently expressed as
\begin{align*}
\mb{z}_0 & = \bmx{\mb{0}\\
\beta\,\mb u_{0}
}\,,\text{ and } & {\mb z}_{t} & =\bmx{\nabla f\left(\tg{\mb C}\mb z_{t-1}\right)\\
\beta\,\mb u_{t}
}\,,\text{ for } t\ge 1\,.
\end{align*}
Our goal is effectively to estimate $\tg{\mb C}$, from the observations $\mb z_0,\mb z_1,\dotsc,\mb z_T$.

Not surprisingly, further model assumptions are needed to exclude inherently intractable instances of the problem. Below in Section \ref{ssec:assumptions} we state the assumptions we make to analyze the problem.

\subsection{Related work}
 Recurrent Neural Networks (RNN) and similar models of random dynamical systems have become the main tool in machine learning applications dealing with sequential data. In this section we briefly review some recent results that provide theoretical analysis for these models.

  Parameter estimation in discrete-time \emph{linear} dynamical systems whose state variable are generally governed by the recursion
 \begin{align}
    \mb x_t &= \tg{\mb A} \mb x_{t-1} +\tg{\mb B} \mb u_{t-1} + \mb \xi_t\,, \label{eq:LDS}
 \end{align}
 with $\xi_t$ denoting an additive observation noise, are studied in \citep{Hardt2018-Gradient, Faradonbeh2018-Finite, Simchowitz2018-Learning, Du2018-HowMany, Sarkar2019-Near}. The difference in the mentioned results stem from variations to the model such as
 \begin{itemize}
    \item observing the state variable \emph{indirectly}, through the sequence
        \begin{align}
            {\mb y}_t & = \tg{\mb A}'{\mb x}_t + \tg{\mb B}'{\mb u}_t + \mb \xi_t' \label{eq:LDS_output}\,,
        \end{align}
        with $\xi'_t$ denoting an additive output noise,

    \item observing \emph{single versus multiple trajectories},

    \item restricting $\tg{\mb A}$ (e.g., $\max_i|\lambda_i(\tg{\mb A})|<1$ in the \emph{stable model} versus $\min_i|\lambda_i(\tg{\mb A})|>1$ in the \emph{explosive model}), and

    \item choosing to have an input (i.e., $\tg{\mb B}=\mb 0$ versus $\tg{\mb B}\ne\mb 0$) with a certain distribution.
 \end{itemize}

 \cite{Hardt2018-Gradient} consider a prediction problem in \emph{controllable} linear dynamical systems with indirect observations as in \eqref{eq:LDS_output}. Specifically, formulating the prediction problem naturally as a (non-convex) least squares, the prediction error achieved by stochastic gradient descent (SGD) is analyzed under some technical assumptions. In a stable \emph{single input single output} setting, it is shown that with $N$ trajectories of length $T$ observed, the prediction error can be bounded as $O(\sqrt{(n^5 +\sigma^2n^3)/(TN)})$ where $n$ is the number of controllable parameters, and $\sigma^2$ is the variance of the zero-mean noise terms $\xi'_t$ in \eqref{eq:LDS_output}.

 Under technical assumptions, \cite{Faradonbeh2018-Finite} establish a sample complexity for estimation of $\mb A_\star$ in the explosive regime (i.e., $\min_i |\lambda_i(\tg{\mb A})|>1)$ with heavy-tailed noise and deactivated input (i.e., $\tg{\mb B}=\mb 0$).

 \cite{Simchowitz2018-Learning} analyze the ordinary least squares (OLS) in estimation of $\tg{\mb A}$ from a single trajectory of observations $\mb x_0,\mb x_1,\ldots$ where there is no input (i.e., $\tg{\mb B}=\mb 0$) and the process noise is i.i.d. samples of a zero-mean isotropic Gaussian random variable. It is shown in \cite{Simchowitz2018-Learning} that for ``marginally stable'' systems (i.e., $\max_i|\lambda_i(\tg{\mb A})| \le 1$), the estimate $\hat{\mb A}$ produced by the OLS, with high probability, achieves the natural error rate of $\norm{\hat{\mb A} - \tg{\mb A}}\lesssim \sqrt{n/T}$ up to some constants and $\log$ factors depending implicitly on $\tg{\mb A}$. Remarkably, this result applies to systems where the spectral radius of $\tg{\mb A}$ equals one (i.e., $\max_i|\lambda_i(\tg{\mb A})|= 1$) where the more standard arguments based on mixing time which require stability of the system do not apply.

 \cite{Oymak2018-Nonasymptotic} considers estimation from a single trajectory of input/ouput observation pairs $(\mb{u}_0,\mb{y}_0),(\mb{u}_1,\mb{y}_1),\dotsc$ where the output sequence $\mb y_0,\mb y_1,\dotsc$ is generated by the recursion \eqref{eq:LDS_output}. Assuming the input, the state noise, and the output noise each to have i.i.d. samples from zero-mean isotropic Gaussian distributions, \citep{Oymak2018-Nonasymptotic} studies accuracy of a least squares approach in estimation of the parameter matrix $\tg{\mb G}=\bmx{\tg{\mb B}' & \tg{\mb A}'\tg{\mb B} & \tg{\mb A}'\tg{\mb A}\tg{\mb B} & \dotsm & \tg{\mb A}'\tg{\mb A}^{T-2}\tg{\mb B}}$ that characterizes the dynamics.

 Similar to \citep{Simchowitz2018-Learning},  \cite{Sarkar2019-Near}  establish the estimation error rate for OLS in the single observation trajectory regime under the model \eqref{eq:LDS} with deactivated input (i.e., $\tg{\mb B}=\mb 0$) and sub-Gaussian noise $\mb \xi_t$. Particularly, in the three regimes of stable or marginally stable systems (i.e. $\max _i|\lambda_i(\tg{\mb A})| < 1+O(1/T)$), marginally stable systems (i.e. $\max _i|\lambda_i(\tg{\mb A})| < 1-O(1/T)$), and explosive systems (in the sense that $\min _i|\lambda_i(\tg{\mb A})| > 1+O(1/T)$) the operator norm of the error roughly decays as $1/\sqrt{T}$, $1/T$, and $e^-T$, respectively.

 \cite{Du2018-HowMany} study the minimax rate of estimation from multiple trajectories in simple linear recurrent neural networks (and convolutional neural networks). Considering the state variable to be linearly collapsed to a scalar in the output, under a subGaussian model for the input sequence as well as the output noise, the mentioned paper provides upper and lower bounds for the minimax risk of the mean squared error. In particular, it is shown that the minimax rate of estimating from $T$ trajectories of length $L$, is orderwise between $\sqrt{\min\{n,L\}p/T}$ and $\sqrt{(p+L)\min\{n,p\}\log(Lp)/T}$.

 From a technical point of view, linearity in recurrent models typically provides the convenience of ``unfolding'' the state recursion into explicit equations in terms of the past input. This convenient feature disappears immediately as nonlinearities are introduced in the recursion as in \eqref{eq:nonlinear-dynamics}. \cite{ Miller2018-Stable} showed that in the stable regime nonlinear RNNs can be approximated by ``truncated'' RNNs. Furthermore, they showed that, for unstable RNNs, gradient descent does not necessarily converge. \cite{Oymak2019-Stochastic} studies the parameter estimation under \eqref{eq:nonlinear-dynamics} when the nonlinearity $\nabla f$ is replaced by an activation function that is \emph{strictly increasing} and applies coordinatewise. Formulating the problem as nonconvex least squares, \citep{Oymak2019-Stochastic} establishes a sample complexity for the convergence of in Frobenius norm. Basically, \citep{Oymak2019-Stochastic} shows that if $T\gtrsim \rho (n+p)$ with $\rho$ being a certain notion of condition number of $\tg{\mb B}$, then with high probability, SGD converges at a linear, albeit dimension dependent, rate.

 In this paper, we generalize the results of \citep{Oymak2019-Stochastic} in two directions. First, our formulation of the recurrence \eqref{eq:nonlinear-dynamics} admits a broader class of nonlinearities, and, as will be seen in the sequel, it enables us to formulate a \emph{convex program} as the estimator. Second, the analysis of \citep{Oymak2019-Stochastic} relies critically on the assumption that the input distribution is Gaussian. This is partly due to the use of the Gaussian concentration inequality for Lipschitz functions. At the cost of having a stricter form of nonlinearity, we relax the requirement on the input distribution by allowing the random input to have heavier tail.

\subsection{Proposed Estimator}
Our proposed estimator is formulated as a convex program as follows
\begin{align}
    \hat{\mb{C}} \in \argmin_{\mb{C}\in\mbb{R}^{n\times (n+p)}} & \sum_{t=1}^T f(\mb{C}\mb{z}_{t-1})-\inp{\mb{x}_t,\mb{C}\mb{z}_{t-1}}\,. \label{eq:estimator}
\end{align}
Readers familiar with convex analysis may observe that if $f_*$, the \emph{convex conjugate} of $f$, is smooth, then  $\nabla f_* \equiv {(\nabla f)}^{-1}$ and  \eqref{eq:nonlinear-dynamics} is equivalent to
\[\nabla f_*(\mb x_{t+1}) = \tg{\mb A} \mb x_{t-1}+\tg{\mb B} \mb u_{t-1}\,.\] Should $\nabla f_*$ be easy to compute, it is evident that the resulting system of linear equations can be solved by the common least squares approach to estimate $\tg{\mb A}$ and $\tg{\mb B}$. However, we prefer \eqref{eq:estimator} as the estimator, since it can be implemented regardless of $f_*$ and its properties.

In view of \eqref{eq:nonlinear-dynamics} and convexity of $f$, it is straightforward to verify that $\mb{C}_\star$ is a minimizer for \eqref{eq:estimator}. Under the assumptions specified below in Section \ref{ssec:assumptions}, we will show that the minimizer of \eqref{eq:estimator} is unique and therefore $\hat{\mb{C}}=\tg{\mb{C}}$. In particular, with $f$ assumed to be $\lambda$-strongly convex, we have
\begin{align*}
    f(\mb{C}\mb{z}_{t-1})-\inp{\mb{x}_t,\mb{C}\mb{z}_{t-1}} & \ge f(\tg{\mb{C}}\mb{z}_{t-1})-\inp{\mb{x}_t,\tg{\mb{C}}\mb{z}_{t-1}} + \frac{\lambda}{2}\norm{(\mb{C}-\tg{\mb{C}})\mb{z}_{t-1}}_2^2\,.
\end{align*}
Therefore, to guarantee uniqueness of the minimizer in \eqref{eq:estimator}, it suffices to show that, with high probability, the smallest eigenvalue of
\begin{align}
\mb{\varSigma} & \defeq\sum_{t=0}^{T-1}\mb{z}_t\mb{z}_t^\T\,,\label{eq:covariance}
\end{align}
 is strictly positive with high probability.

 \subsection{Assumptions}\label{ssec:assumptions}
With no restricting conditions imposed on the observation model in \eqref{eq:nonlinear-dynamics}, the posed estimation problem is not meaningful. For instance, any affine function $f$ is permitted in the core model above, but clearly its corresponding trajectory conveys no information about $\mb{C}_\star$.

\begin{asm}[regularity of $f$] \label{asm:regularity-f} The function $f$ has the following properties:
    \begin{enumerate}
        \item The function $f$ is $\lambda$-strongly convex and $\varLambda$-smooth in the usual sense, i.e.,
            \begin{align}
                \frac{\lambda}{2}\norm{\mb{y}-\mb{x}}_2^2 &\le f(\mb{y}) - f(\mb{x}) - \inp{\nabla f(\mb{x}),\mb{y}-\mb{x}} \le \frac{\varLambda}{2}\norm{\mb{y}-\mb{x}}_2^2\,, \label{eq:convex-smooth}
            \end{align}
            holds for all $\mb{x},\mb{y}\in \mbb{R}^n$.
        \item There exist a matrix-valued function $F\st \mbb{R}^n\to \mbb{R}^{n\times n}$ and a relatively small constant $\varepsilon>0$ such that, for all $\mb{x},\mb{y}\in\mbb{R}^n$, we have
            \begin{align}
                \norm{\frac{1}{2}\left(\nabla f(\mb{x}+\mb{y}) - \nabla f(\mb{x}-\mb{y})\right) - F(\mb{x})\mb{y} }_2 \le \varepsilon \norm{\mb{y}}_2\,. \label{eq:local-approx-linearity}
            \end{align}
    \end{enumerate}
\end{asm}

Perhaps the simplest example of the functions that meet the conditions of Assumption \ref{asm:regularity-f} is the convex quadratic functions. Let $\mb Q$ be a positive semidefinite matrix that satisfies $\lambda \mb I \preceq \mb Q\preceq \varLambda \mb I$. Then $f(\mb x) = \frac{1}{2}\mb x^\T \mb Q \mb x$ clearly satisfies \eqref{eq:convex-smooth}, and also satisfies \eqref{eq:local-approx-linearity} for $F(\mb x) \equiv \mb Q$ and $\varepsilon = 0$.

Another example of the function $f$ that meets the above conditions is the piecewise quadratic function
\[f(\mb x) = \frac{1}{2}\sum_{i=1}^n \max\{\lambda (-x_i)^2_+,\varLambda (x_i)_+^2\}\,.\] The gradient of this function can be written as
\[
    \nabla f(\mb x) =\bmx{
                        \frac{\varLambda+\lambda}{2}x_1+\frac{\varLambda - \lambda}{2}|x_1|\\
                        \vdots\\
                        \frac{\varLambda+\lambda}{2}x_n+\frac{\varLambda - \lambda}{2}|x_n|
                        }\,,
\]
whose coordinates happen to be the (parameterized) ReLU functions. For this specific $f$, the mapping $F$ can be chosen as
\[
    F(\mb x) = \bmx{\frac{\varLambda+\lambda}{2} + \frac{\varLambda-\lambda}{2}\sgn(x_1) & & & \\
                    & \frac{\varLambda+\lambda}{2} + \frac{\varLambda-\lambda}{2}\sgn(x_2) & & \\
                    & & \ddots & \\
                    & & & \frac{\varLambda+\lambda}{2} + \frac{\varLambda-\lambda}{2}\sgn(x_n)
                    }\,,
\]
for which \eqref{eq:local-approx-linearity} holds with $\varepsilon = (\varLambda-\lambda)/2$.

An immediate consequence of Assumption \ref{asm:regularity-f} is the following.
\begin{lem}\label{lem:bound-F}
    Under Assumption \ref{asm:regularity-f}, the mapping $F$ obeys
    \[(\lambda -\varepsilon)\norm{\mb{y}}_2\le \norm{F(
    \mb{x})\mb{y}}_2 \le (\varLambda +\varepsilon)\norm{\mb{y}}_2\,,\]
    for all $\mb{x},\mb{y}\in\mbb{R}^n$.
\end{lem}
\begin{proof}
    Using the standard equivalent definitions of strong convexity and smoothness \citep[Theorem 2.1.5]{Nesterov2013-Introductory}, we have
    \[\lambda \norm{\mb{y}-\mb{x}}_2 \le \norm{\nabla f(\mb{y}) - \nabla f(\mb{x})}_2 \le \varLambda \norm{\mb{y}-\mb{x}}_2\,.\]
    Rewriting these inequalities, in terms of the pair $(-\mb{x}+\mb{y},\mb{x}+\mb{y})$ in place of $(\mb{x},\mb{y})$, we can obtain
    \[2\lambda \norm{\mb{y}} \le \norm{\nabla f(\mb{x}+\mb{y})- \nabla f(\mb{x}-\mb{y})}_2\le 2\varLambda \norm{\mb{y}}_2\,.\]
Furthermore, by \eqref{eq:local-approx-linearity} and the triangle inequality we have
\[\norm{F(\mb{x})\mb{y}}_2 - \varepsilon\norm{\mb{y}}_2 \le  \frac{1}{2}\norm{\nabla f(\mb{x}+\mb{y})- \nabla f(\mb{x}-\mb{y})}_2 \le \norm{F(\mb{x})\mb{y}}_2 + \varepsilon\norm{\mb{y}}_2\,.\]
 The lemma easily follows from the latter two lines of inequalities.
\end{proof}

We make the following assumption on the input $\mb u$.
\begin{asm}[regularity of the input distribution] \label{asm:regularity-u} The input $\mb{u}$ has the following properties:
    \begin{enumerate}
        \item The input $\mb u\in\mbb{R}^p$ is a zero-mean isotropic random variable, i.e.,
            \[
                \E(\mb{u}) = \mb{0},\ \text{and}\ \E(\mb{u}\mb{u}^\T) = \mb{I}\,.
            \]
        \item The coordinates of $\mb{u}$ have independent symmetric distributions, i.e., for all measurable subsets $\mc{A}=\mc{A}_1\times\dotsc\times\mc{A}_p$ of $\mbb{R}^p$, we have
            \begin{align*}
                \P(\mb{u}\in\mc{A}) = \prod_{i=1}^p \P(u_i\in\mc{A}_i) &= \prod_{i=1}^p\P(-u_i\in\mc{A}_i)\,.
            \end{align*}
        \item For a certain $\alpha \ge 1$, the input $\mb{u}$ has a bounded directional Orlicz $\psi_\alpha$ norm, i.e., there exists a finite absolute constant $K>0$ such that
            \begin{align}
                \sup_{\mb{h}\in\mbb{S}^{p-1}}\E \left(e^{\nicefrac{|\inp{\mb{h},\mb{u}}|^\alpha}{K^\alpha}}\right) & \le 2\,.\label{eq:Orlicz-psi-alpha}
            \end{align}
    \end{enumerate}
\end{asm}

The following lemma is an immediate consequence of \eqref{eq:Orlicz-psi-alpha}.
\begin{lem}\label{lem:concatenated-moments}
    Let $\mb{u}$ be the random variable under the Assumption \ref{asm:regularity-u} and $\mb{u}'$ be an independent copy $\mb{u}$. The vector $\mb{u}$ has a bounded directional fourth moment, i.e., there exist $\eta \in [1,\,2{(4/\alpha)}^{4/\alpha}K^4]$ such that
    \begin{align}
        \E\left(\left(\inp{\mb{h},\mb{u}}\right)^4\right) &\le \eta \label{eq:directional-fourth-moment}\,,
    \end{align}
    holds for all $\mb{h}\in\mbb{S}^{p-1}$.    Furthermore, for all $\mb{h},\mb{h}' \in \mbb{S}^{p-1}$ we have
    \[\E\left(\left( \inp{\mb{h},\mb{u}} + \inp{\mb{h}',\mb{u}}\right)^4\right) \le \max\{\eta,3\}\left(\norm{\mb{h}}_2^2 + \norm{\mb{h}'}_2^2\right)^2\,.\]
\end{lem}
\begin{proof}
    Clearly, existence of the exponential moments guarantees that $\E\left(|\inp{\mb{h},\mb{u}}|^4\right)<\infty$ for all $\mb{h}\in\mbb{S}^{p-1}$. To prove the first part, we show that \eqref{eq:directional-fourth-moment} holds for $\eta = 2{(4/\alpha)}^{4/\alpha}K^4$. For all $\mb{h}\in\mbb{S}^{p-1}$ we have
    \begin{align*}
        \E\left(|\inp{\mb{h},\mb{u}}|^4\right)  &= \frac{\eta}{2}\, \E\left(\left(\frac{|\inp{\mb{h},\mb{u}}|^\alpha}{(\eta/2)^{\alpha/4}}\right)^{4/\alpha}\right)\\
                                                &\le \frac{\eta}{2}\, \E\left(\exp\left(\frac{4}{\alpha}\,\frac{|\inp{\mb{h},\mb{u}}|^\alpha}{(\eta/2)^{\alpha/4}}\right)\right)\,.
    \end{align*}
    For the prescribed $\eta$ we have $(\eta/2)^{\alpha/4} \alpha/4 =K^\alpha$. Thus, in view of \eqref{eq:Orlicz-psi-alpha}, we obtain
    \begin{align*}
        \E\left(|\inp{\mb{h},\mb{u}}|^4\right) &\le \frac{\eta}{2}\, 2 = \eta\,,
    \end{align*}
    as desired.
    Since $\mb{u}$ and $\mb{u}'$ are zero-mean, isotropic, i.i.d, and further obey \eqref{eq:directional-fourth-moment}, we have
        \begin{align*}
            \E\left(\left( \inp{\mb{h},\mb{u}} + \inp{\mb{h}',\mb{u}}\right)^4\right) &= \E\left(\left( \inp{\mb{h},\mb{u}}\right)^4 + 6\left( \inp{\mb{h},\mb{u}}\right)^2 \left(\inp{\mb{h}',\mb{u}}\right)^2 + \left(\inp{\mb{h}',\mb{u}}\right)^4\right)\\
            & \le \eta \norm{\mb{h}}_2^4  + 6\norm{\mb{h}}_2^2\norm{\mb{h}'}_2^2 +  \eta \norm{\mb{h}'}_2^4\\
            &\le \max\{\eta , 3\}\left(\norm{\mb{h}}_2^2+\norm{\mb{h}'}_2^2\right)^2\,,
        \end{align*}
    which proves the second part.
\end{proof}
 In addition to the assumptions made above, our analysis crucially depends on a form of contraction
that can be ensured by the following assumption. Note that $\varLambda$ can be taken as $\varLambda = \mr{Lip}(\nabla f)$, i.e., the Lipschitz constant of $\nabla f$ with respect to the usual Euclidean metric.
\begin{asm}[conctractive dynamics]
 The nonlinearity $\nabla f(\cdot)$ and the matrix $\tg{\mb A}$ induce a contraction in the sense that \[\varLambda \,\norm{\tg{\mb A}}<1\,.\]
\end{asm}

\section{Main result}
Our main theorem below, effectively guarantees that $T=\tilde{O}(n+p)$ is sufficient for the matrix $\mb \varSigma$ to be (strictly) positive definite.

\begin{thm}\label{thm:main}
    Suppose that the energy of $\tg{\mb B}$ is well-spread among its columns in the sense that $\norm{\tg{\mb B}}_{1\to 2}/\norm{\tg{\mb B}}_\F = O(p^{-1/2})$. Furthermore, suppose that the constant
    \begin{equation}
     \begin{aligned}
         \theta = \theta_{\alpha,\beta,\varepsilon, \lambda, K,\tg{\mb B}} &\defeq -\varepsilon K\log^{\frac{1}{\alpha}}\left(10\max\{\eta,3\}\right)\norm{\tg{\mb{B}}}_{1\to 2}\\
         &\hphantom{\defeq}+ 0.6\min_{i=1,\dotsc,p}\min\left\{\beta, (\lambda -\varepsilon) \lambda^{1/2}_{\min}\left({\tg{\mb B}}_{\backslash i}{\tg{\mb B}}^\T_{\backslash i}\right)\right\} \,,
     \end{aligned}  \label{eq:theta*}
 \end{equation}
 is strictly positive. Furthermore, suppose that $L$ satisfies
    \begin{align}
            L  \ge 1+\frac{\log\left(\frac{c^2}{\theta^2}\log\left(\frac{2(T-1)(p+1)}{\delta}\right)\left(\frac{\varLambda\norm{\tg{\mb B}}_{\F}}{1-\varLambda\norm{\tg{\mb A}}}\right)^2\right)}{\log\frac{1}{\varLambda\norm{\tg{\mb A}}}}\label{eq:L-main}\,,
    \end{align}
    for a sufficiently large constant $c>0$. Then, for
    \begin{align}
        T \gtrsim \max\{\eta^2,9\}\,(n+p)L\log\left(\frac{eT/L}{n+p}\right)+\log\left(\frac{8L}{\delta}\right)\label{eq:T-main}\,,
    \end{align}
    we have
    \[
        \lambda_{\min}(\mb \varSigma) \gtrsim \frac{\theta^2}{\max\{\eta,3\}}T\,,
    \]
    with probability $\ge 1-\delta$. Consequently, on the same event, \eqref{eq:estimator} recovers $\tg{\mb C}$, exactly.
\end{thm}

A critical condition of Theorem \ref{thm:main} is that $\theta$ is strictly positive. This condition implicitly requires $\varepsilon$ in \eqref{eq:local-approx-linearity} to be sufficiently small, which in turn implies the condition number of $f$ is sufficiently close to $1$ (i.e. $\nabla f$ is nearly linear). Furthermore, it is needed that the energy of $\mb B_\star$ to be well-spread not only among its columns, but also in a ``spectral'' sense. More precisely, we need the quantity
\[
    \max_{i=1,\dotsc,p}\frac{\norm{\tg{\mb B}}_{1\to 2}}{\lambda_{\min}^{1/2}({\tg{\mb B}}_{\backslash i}{\tg{\mb B}}_{\backslash i}^\T)}\,,
\]
to be sufficiently small.
The equation \eqref{eq:theta*} also suggests that a reasonable choice of the normalizing constant $\beta$ should satisfy $\beta \approx (\lambda -\varepsilon)\lambda_{\min}^{1/2}\left(\mb B_\star\mb B_\star^\T\right)$.
\section{Simulation}
We evaluated the proposed estimator numerically on synthetic data in a setup similar to the experiments of \citep{Oymak2019-Stochastic}. In all of the experiments, we consider the dimensions to be $n = 50$, $p=100$, and the time horizon to be $T=500$. For $\alpha \in \{0.2,\,0.8\}$ we choose $\tg{\mb A}=\alpha \mb R$ with $\mb R$ being a uniformly distributed $n\times n$ orthogonal matrix. Furthermore, $\tg{\mb B}\in\mbb{R}^{n\times p}$ is generated randomly with i.i.d. standard normal entries. The normalizing factor is chosen as $\beta = $ following  what prescribed in \citep{Oymak2019-Stochastic}. We consider two different models for the input $\mb u$. Let $g\sim\mr{Normal}(0,1)$ denote a standard Normal random scalar. The first model is similar to the model of \citep{Oymak2019-Stochastic} where the entries of $\mb u$ are i.i.d. copies of $g$, whereas in the second model takes i.i.d. copies of $g^3$ as the entries of $\mb u$. We refer to these models as the Gaussian model and the heavy-tailed model, respectively.
The nonlinearity in \eqref{eq:nonlinear-dynamics} is described by one of the functions
    \[f(\mb x) = \frac{1-\rho}{2}\sum_{i=1}^n (x_i)_+^2 + \frac{\rho}{2}\sum_{i=1}^n x_i^2\,,\]
at $\rho = 1$ (i.e., linear activation), $\rho = 0.5$ (i.e., leaky ReLU activation with slope $0.5$ over $\mbb R_{\le 0}$), $\rho = 0.3$ (i.e., leaky ReLU activation with slope $0.3$ over $\mbb R_{\le 0}$), and $\rho = 0$ (i.e., ReLU activation).

For each choice of $\alpha$ and $\rho$, we solved \eqref{eq:estimator} using Nesterov's Accelerated Gradient Method (AGM) (\citealp{Nesterov1983-Method}; \citealp[][Section 2.2]{Nesterov2013-Introductory}), for $100$ randomly generated instances of the problem. For the Gaussian model the step-size is set to $10^{-3}$, whereas for the heavy-tailed model the step-size is set to $10^{-4}$. In each trial, the AGM is run for a maximum of $500$ iterations and terminated only if the relative error dropped below $10^{-8}$ (i.e., $\norm{\hat{\mb C} - \tg{\mb C}}_\F^2/\norm{\tg{\mb C}}_\F^2\le 10^{-8}$). The optimization task can be solved by the SGD as well. However, slower convergence of the SGD is only tolerable for large-scale problems where lower memory load is crucial. Nevertheless, because the estimator \eqref{eq:estimator} is formulated as a convex program, we can apply the SGD methods with variance reduction \citep[see e.g.,][]{Johnson2013-Accelerating,Schmidt2017-Minimizing,Defazio2014-SAGA} and rely on their theoretical guarantees.

\begin{figure}
    \centering
    \begin{subfigure}{0.49\textwidth}
            \centering
            \includegraphics[width=\textwidth]{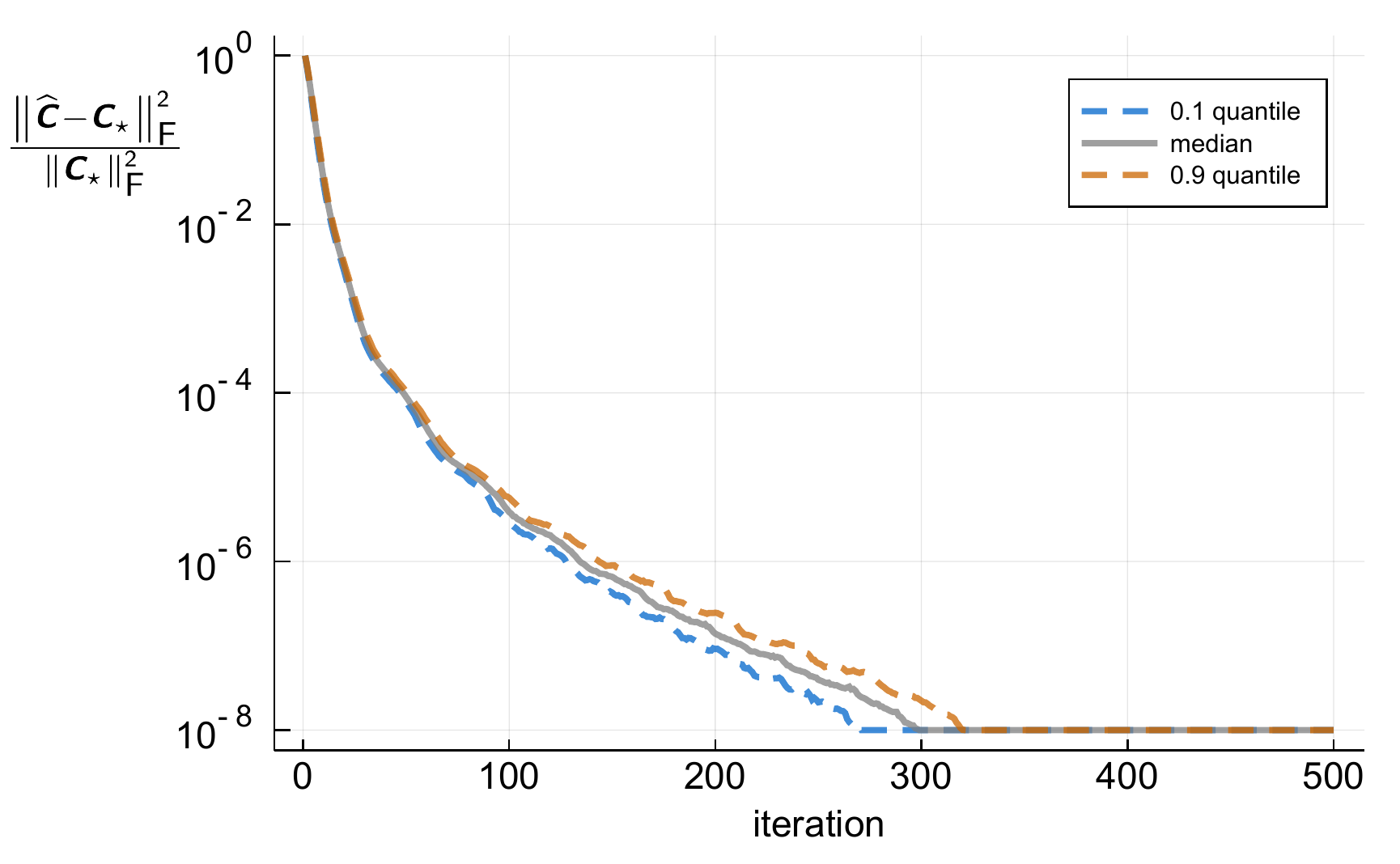}
            \caption{$\alpha = 0.2$, $\rho = 1$}
    \end{subfigure}
    \begin{subfigure}{0.49\textwidth}
            \centering
            \includegraphics[width=\textwidth]{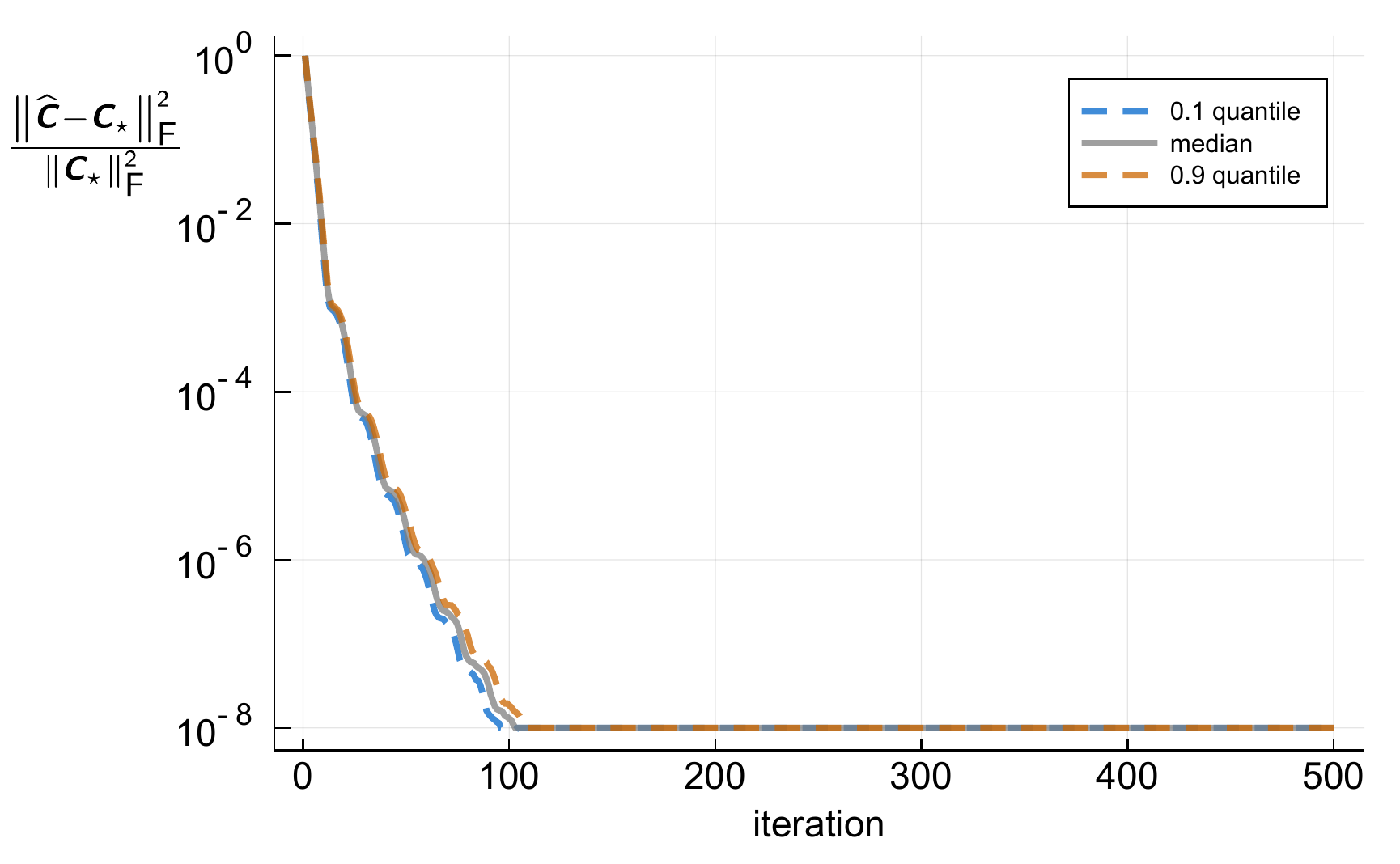}
            \caption{$\alpha = 0.8$, $\rho = 1$}
    \end{subfigure}

    \begin{subfigure}{0.49\textwidth}
            \centering
            \includegraphics[width=\textwidth]{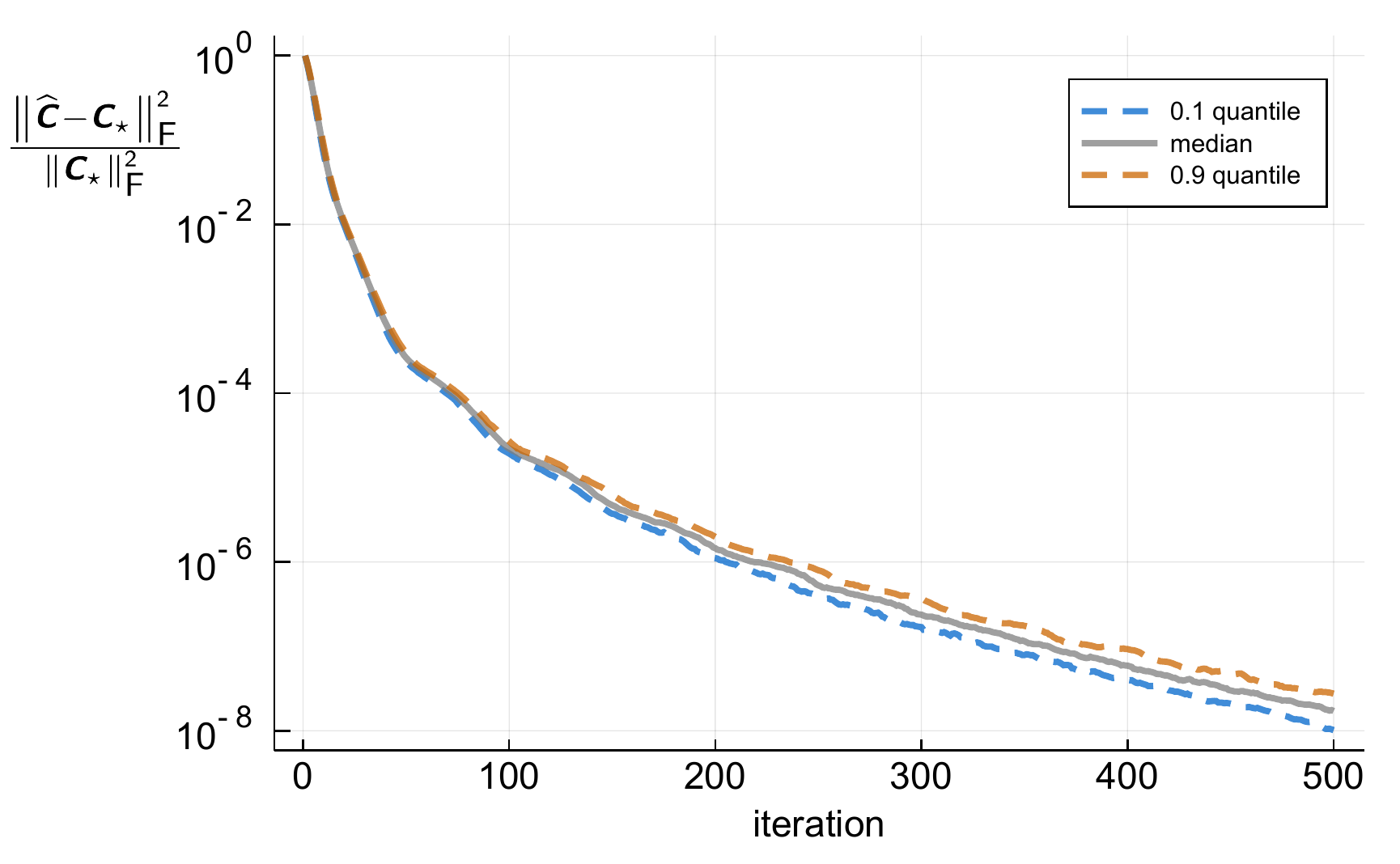}
            \caption{$\alpha = 0.2$, $\rho = 0.5$}
    \end{subfigure}
    \begin{subfigure}{0.49\textwidth}
            \centering
            \includegraphics[width=\textwidth]{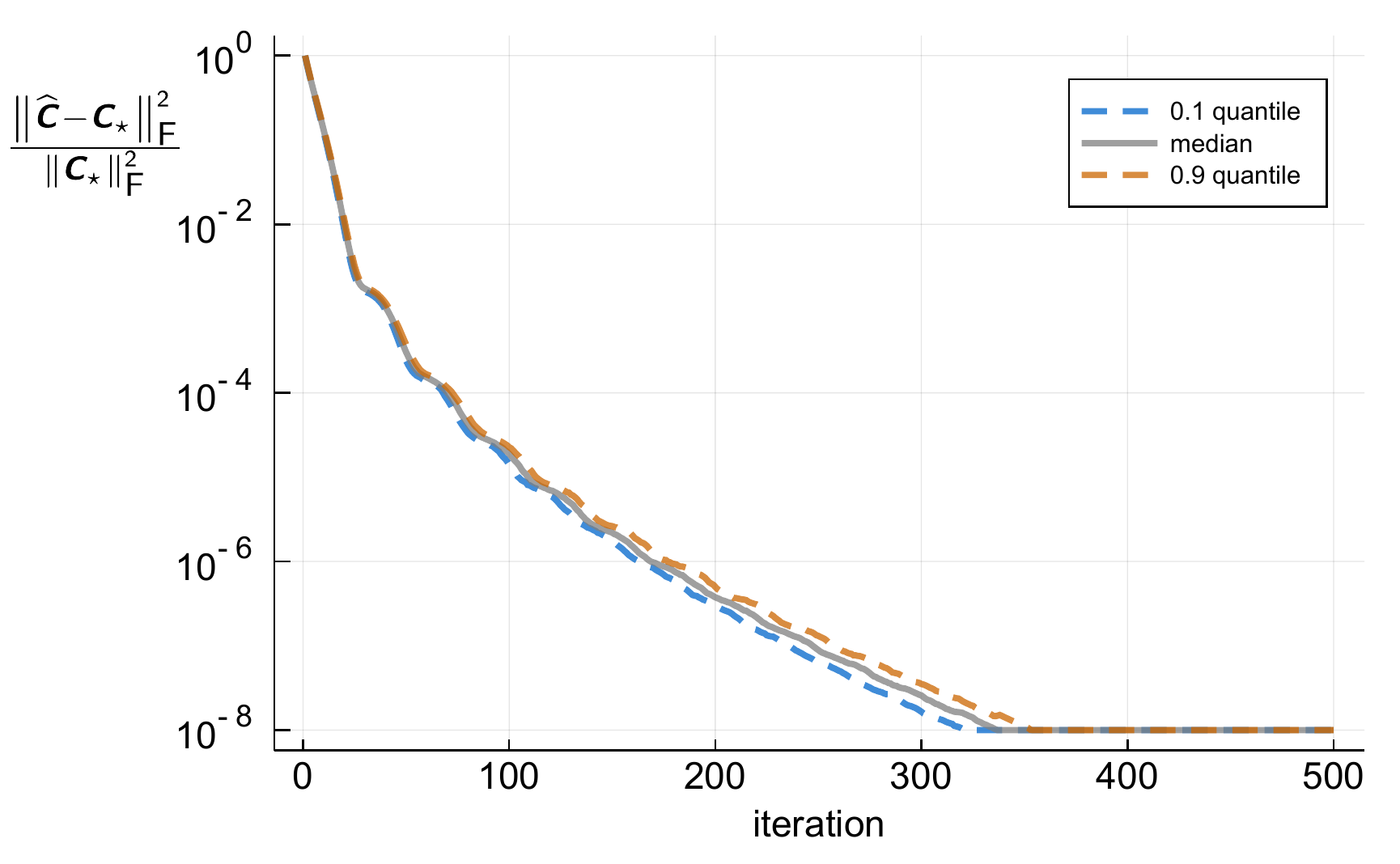}
            \caption{$\alpha = 0.8$, $\rho = 0.5$}
    \end{subfigure}

    \begin{subfigure}{0.49\textwidth}
            \centering
            \includegraphics[width=\textwidth]{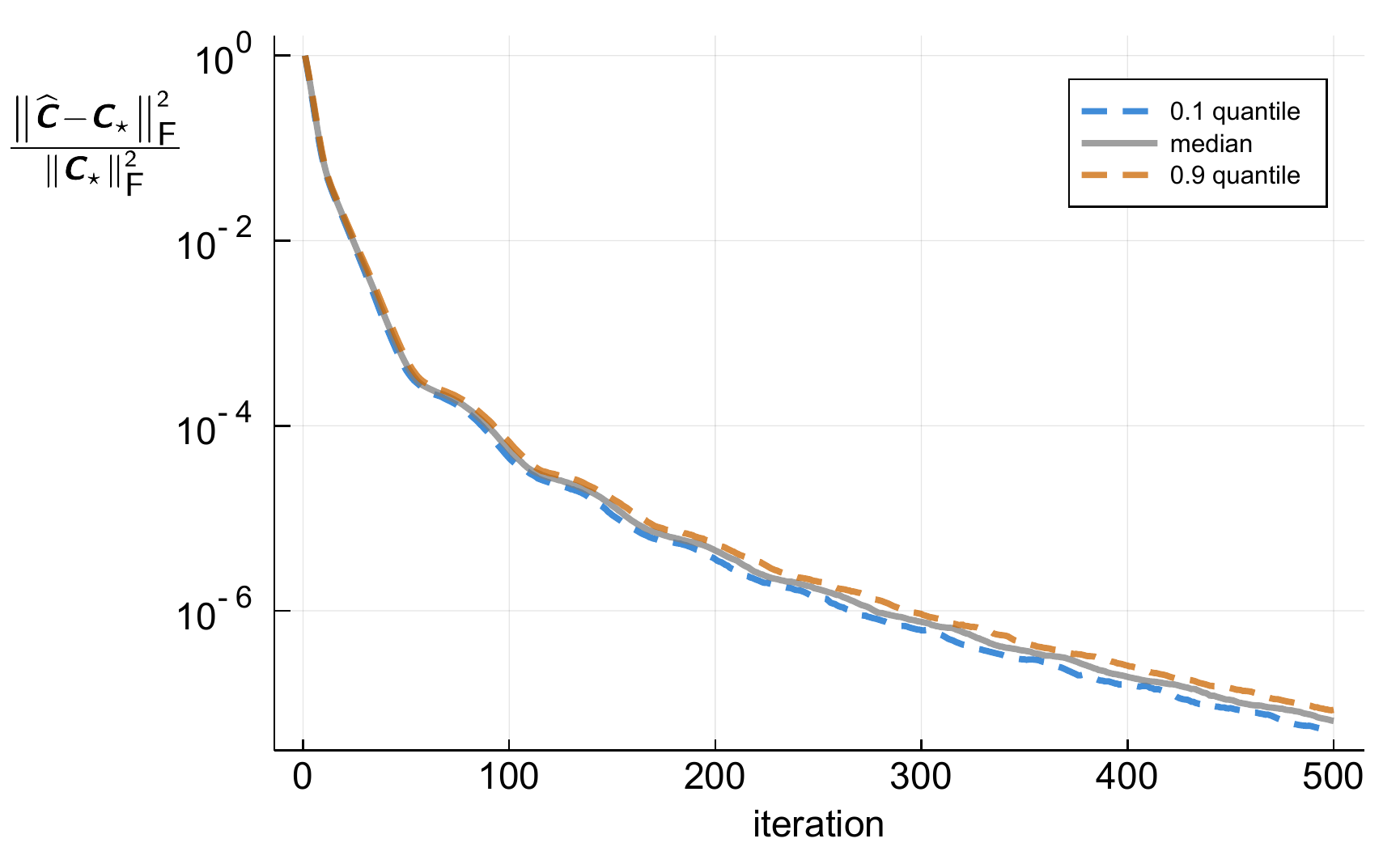}
            \caption{$\alpha = 0.2$, $\rho = 0.3$}
    \end{subfigure}
    \begin{subfigure}{0.49\textwidth}
            \centering
            \includegraphics[width=\textwidth]{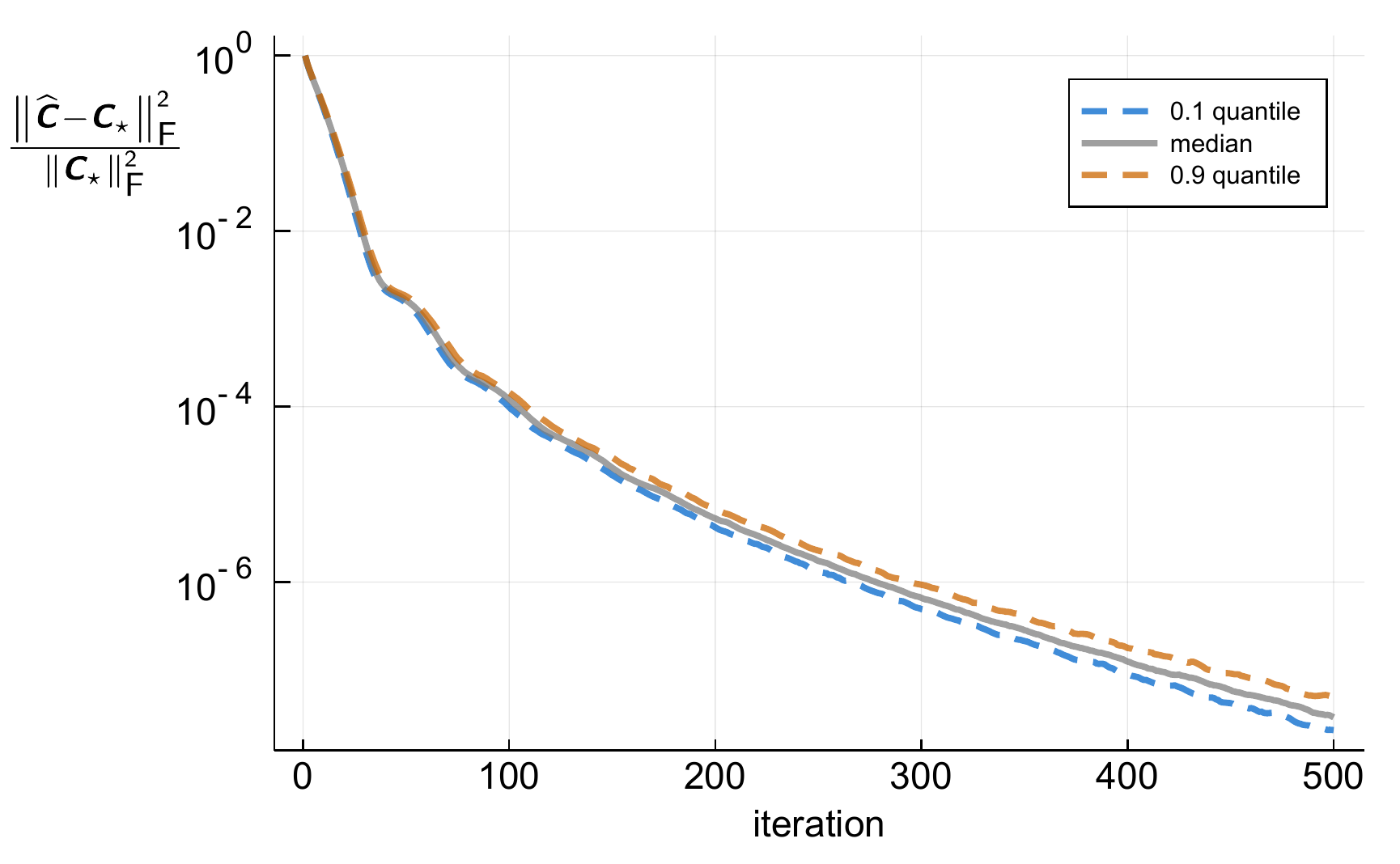}
            \caption{$\alpha = 0.8$, $\rho = 0.3$}
    \end{subfigure}

    \begin{subfigure}{0.49\textwidth}
            \centering
            \includegraphics[width=\textwidth]{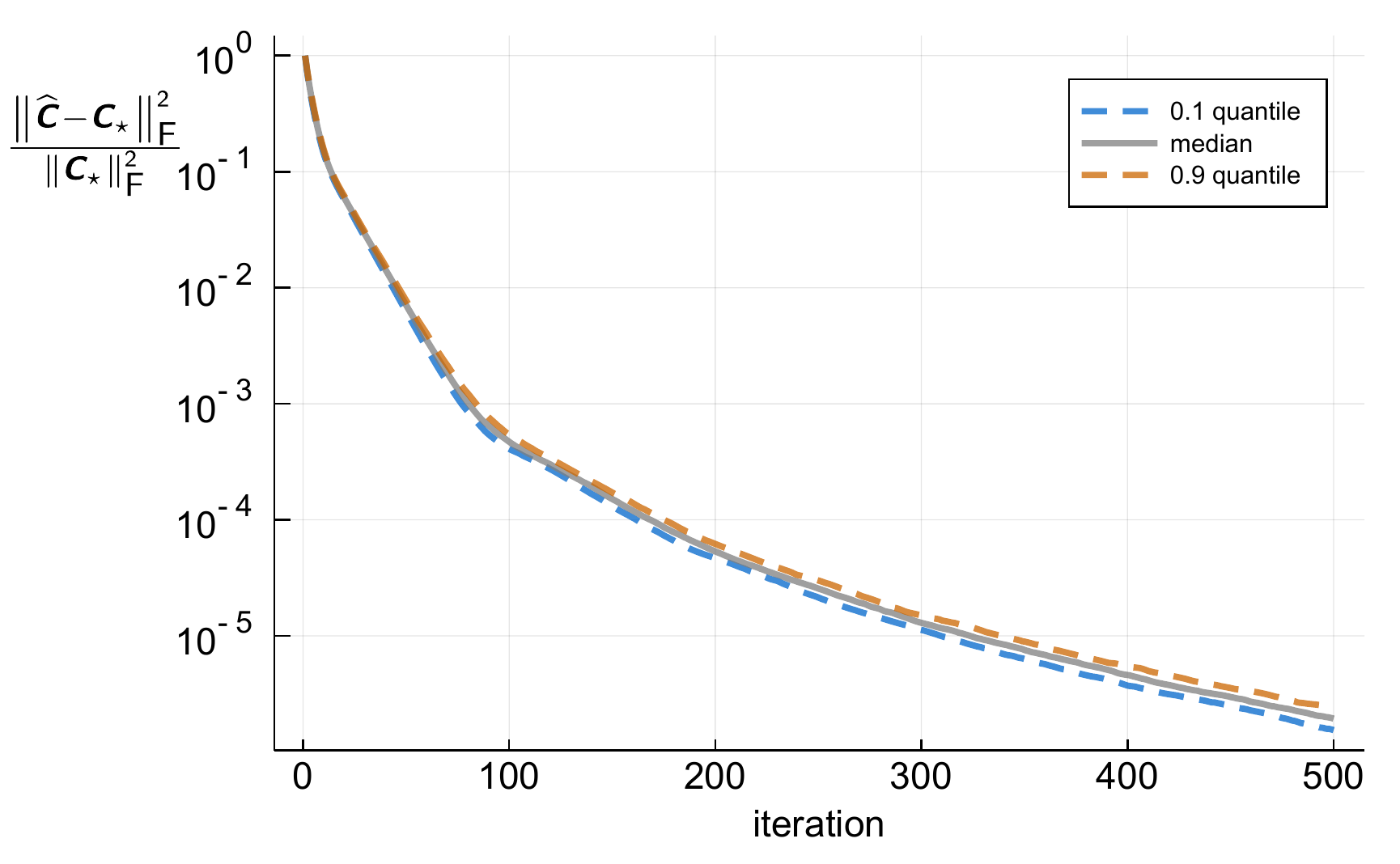}
            \caption{$\alpha = 0.2$, $\rho = 0$}
    \end{subfigure}
    \begin{subfigure}{0.49\textwidth}
            \centering
            \includegraphics[width=\textwidth]{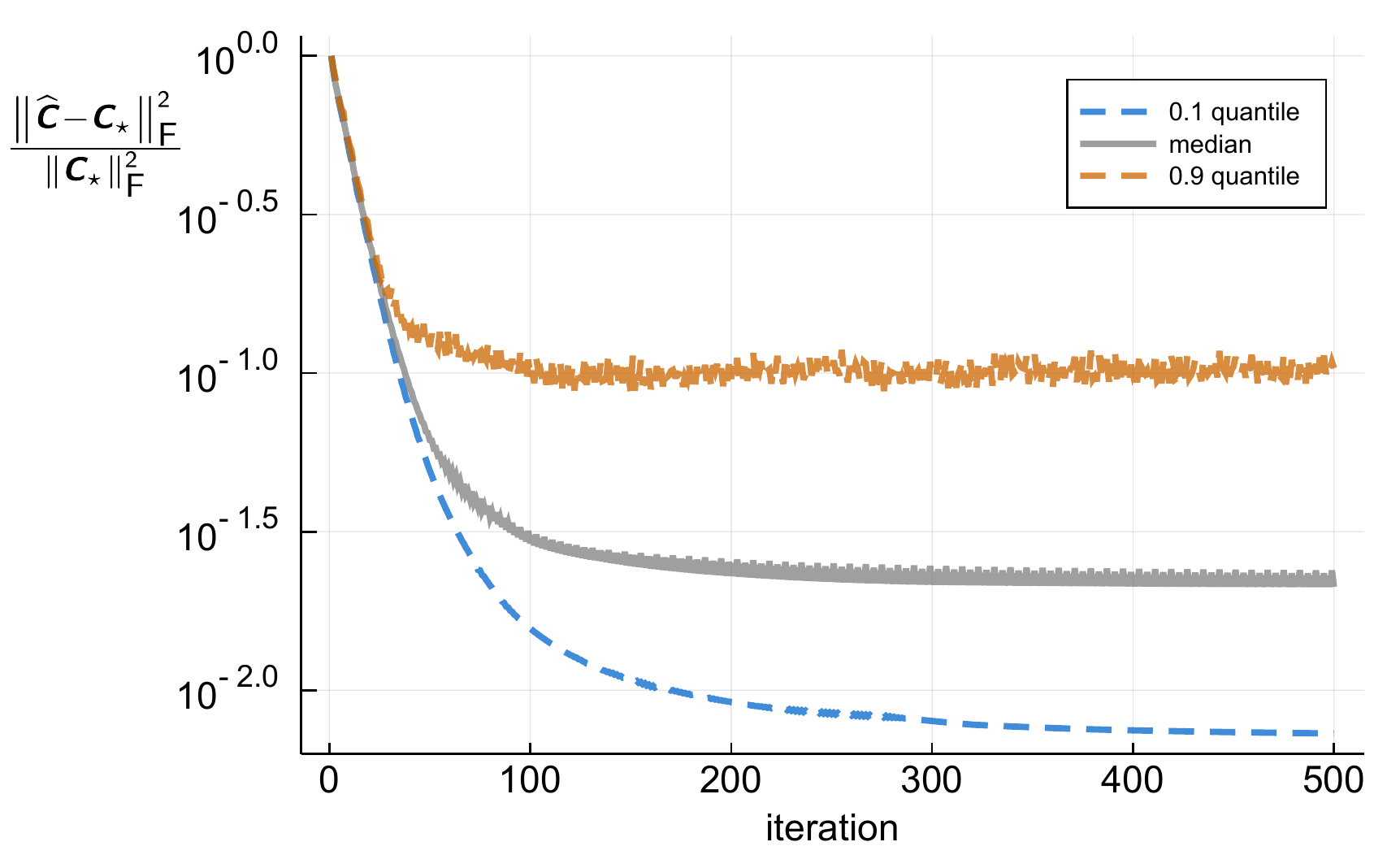}
            \caption{$\alpha = 0.8$, $\rho = 0$}
    \end{subfigure}
    \caption{Gaussian model}
    \label{fig:Gaussian-model}
\end{figure}

\begin{figure}
    \centering
    \begin{subfigure}{0.49\textwidth}
            \centering
            \includegraphics[width=\textwidth]{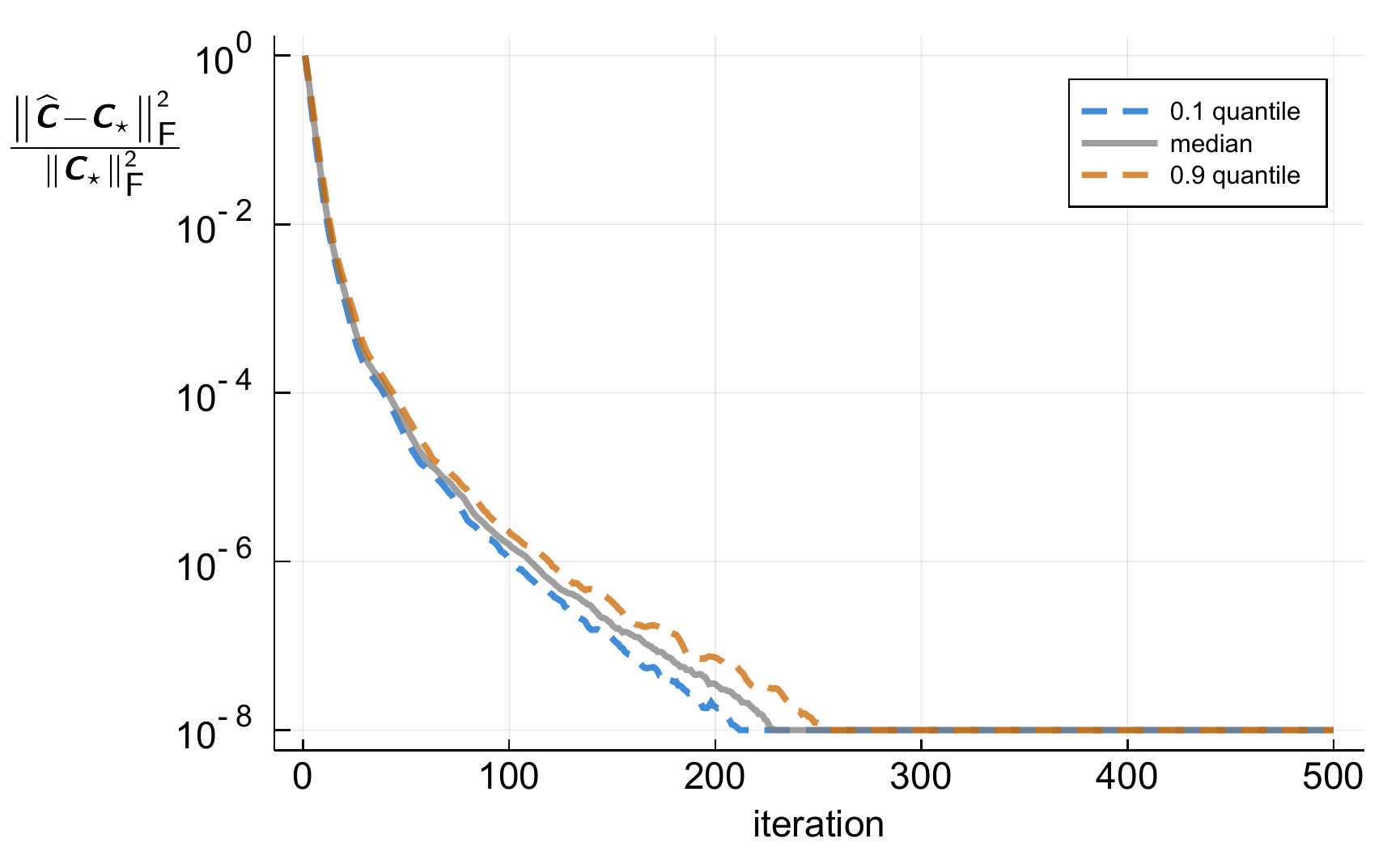}
            \caption{$\alpha = 0.2$, $\rho = 1$}
    \end{subfigure}
    \begin{subfigure}{0.49\textwidth}
            \centering
            \includegraphics[width=\textwidth]{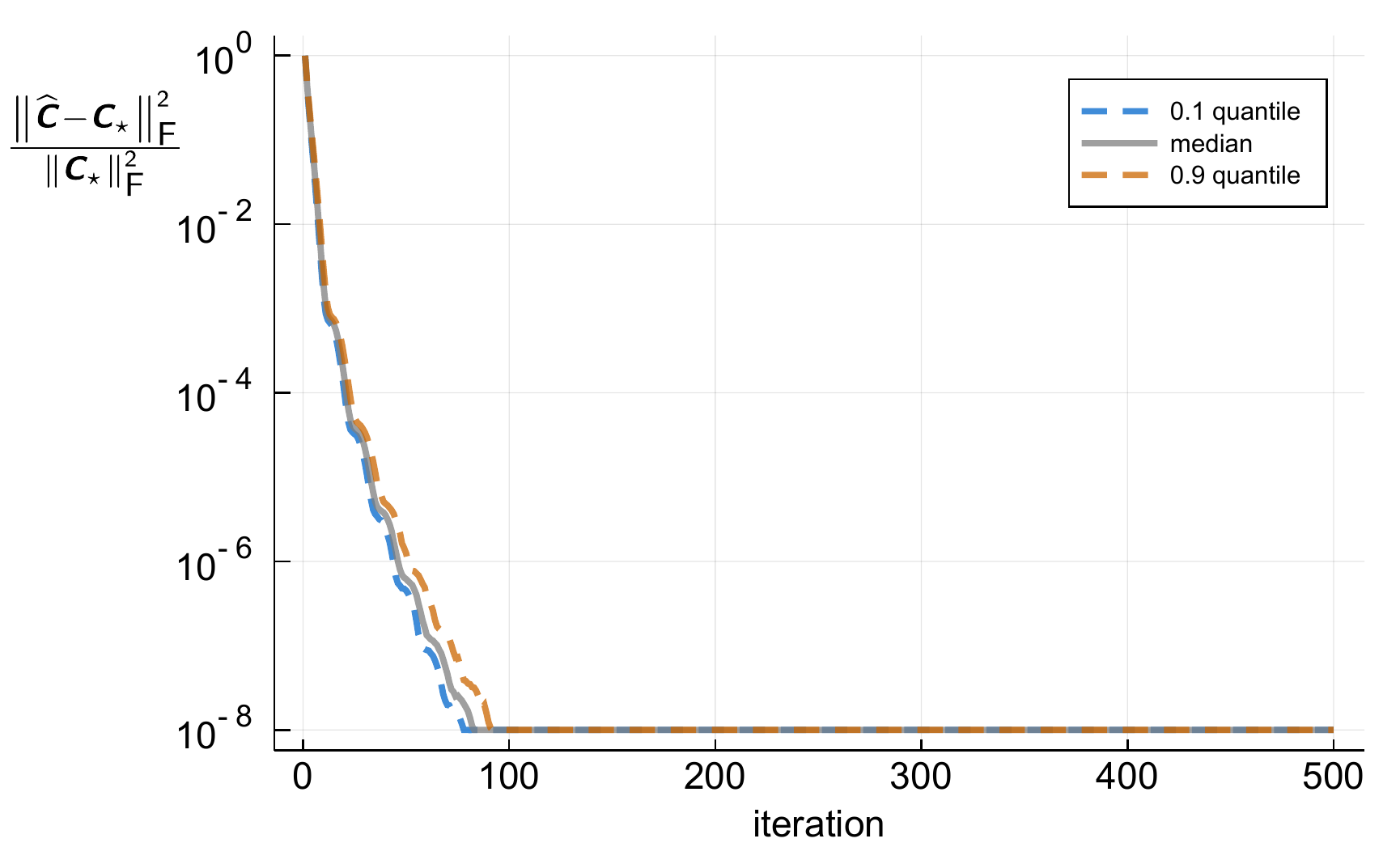}
            \caption{$\alpha = 0.8$, $\rho = 1$}
    \end{subfigure}

    \begin{subfigure}{0.49\textwidth}
            \centering
            \includegraphics[width=\textwidth]{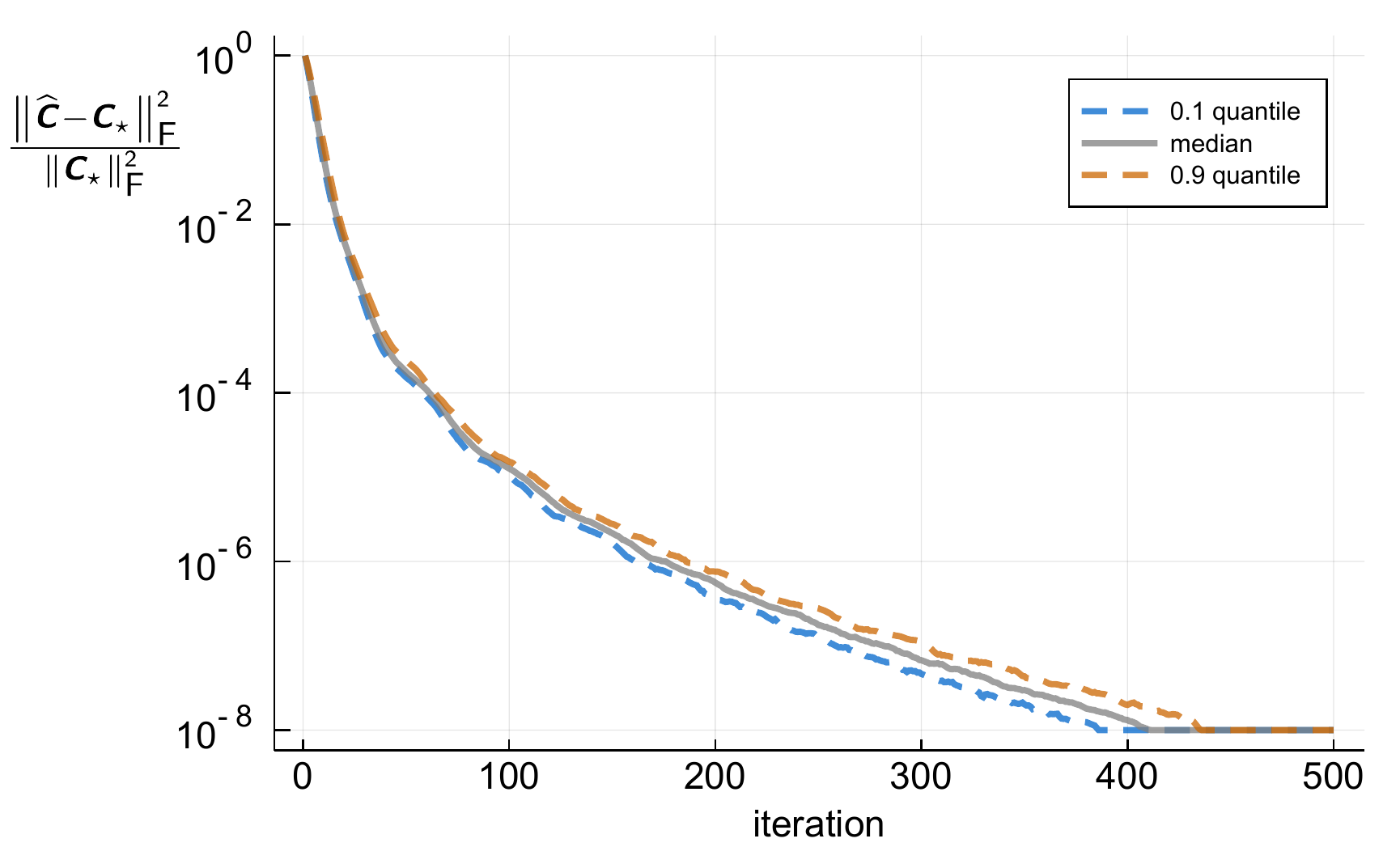}
            \caption{$\alpha = 0.2$, $\rho = 0.5$}
    \end{subfigure}
    \begin{subfigure}{0.49\textwidth}
            \centering
            \includegraphics[width=\textwidth]{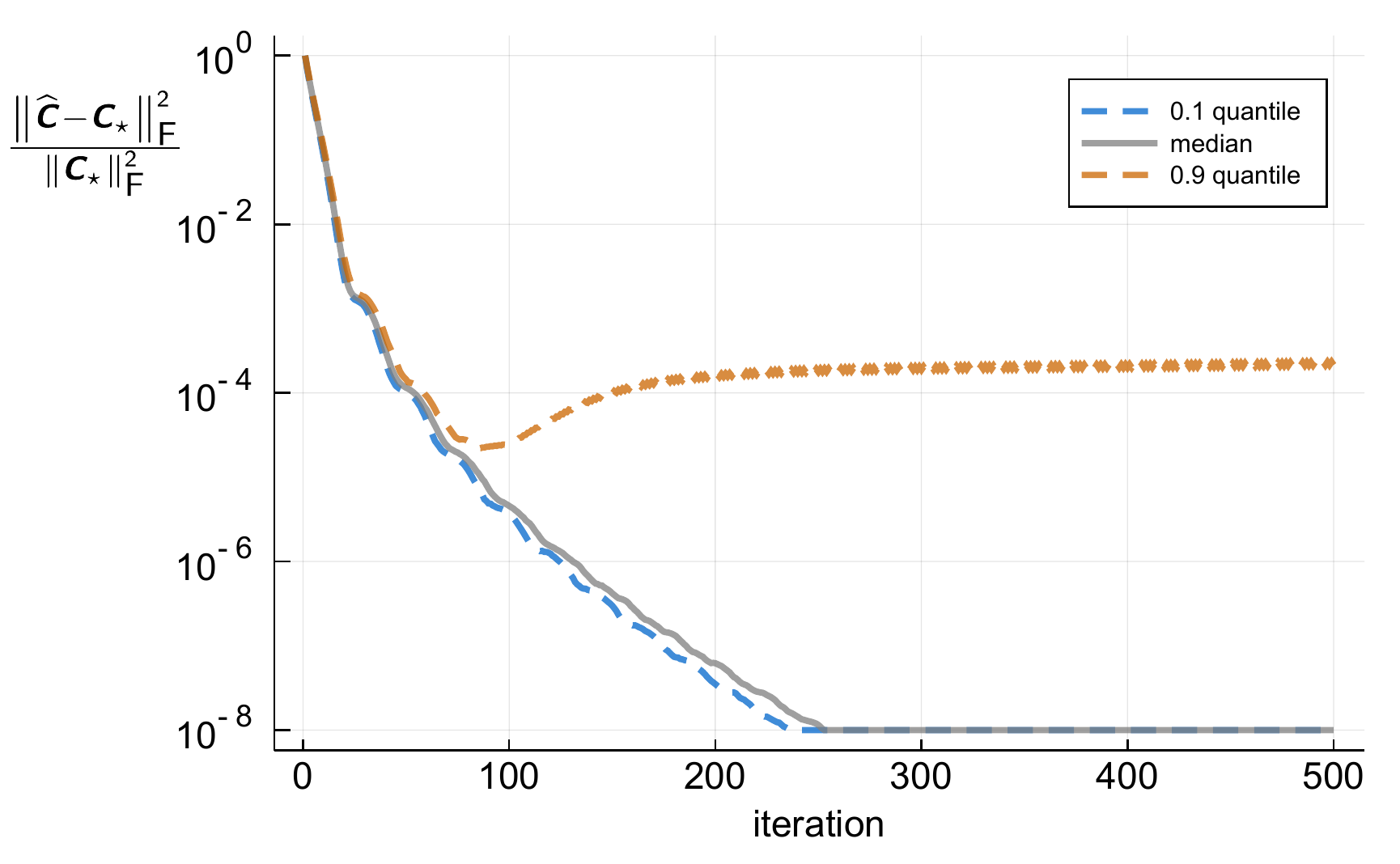}
            \caption{$\alpha = 0.8$, $\rho = 0.5$}
    \end{subfigure}

    \begin{subfigure}{0.49\textwidth}
            \centering
            \includegraphics[width=\textwidth]{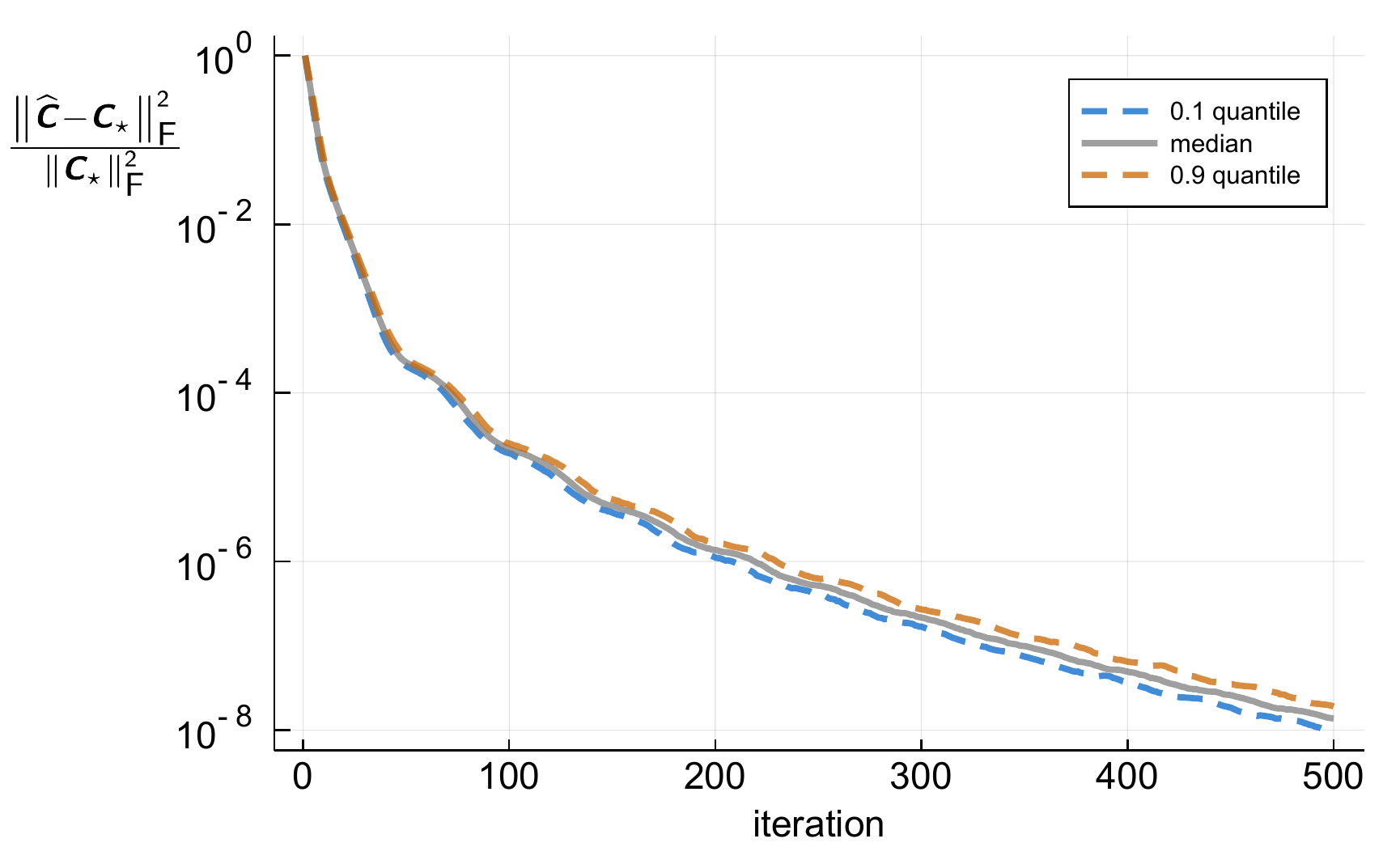}
            \caption{$\alpha = 0.2$, $\rho = 0.3$}
    \end{subfigure}
    \begin{subfigure}{0.49\textwidth}
            \centering
            \includegraphics[width=\textwidth]{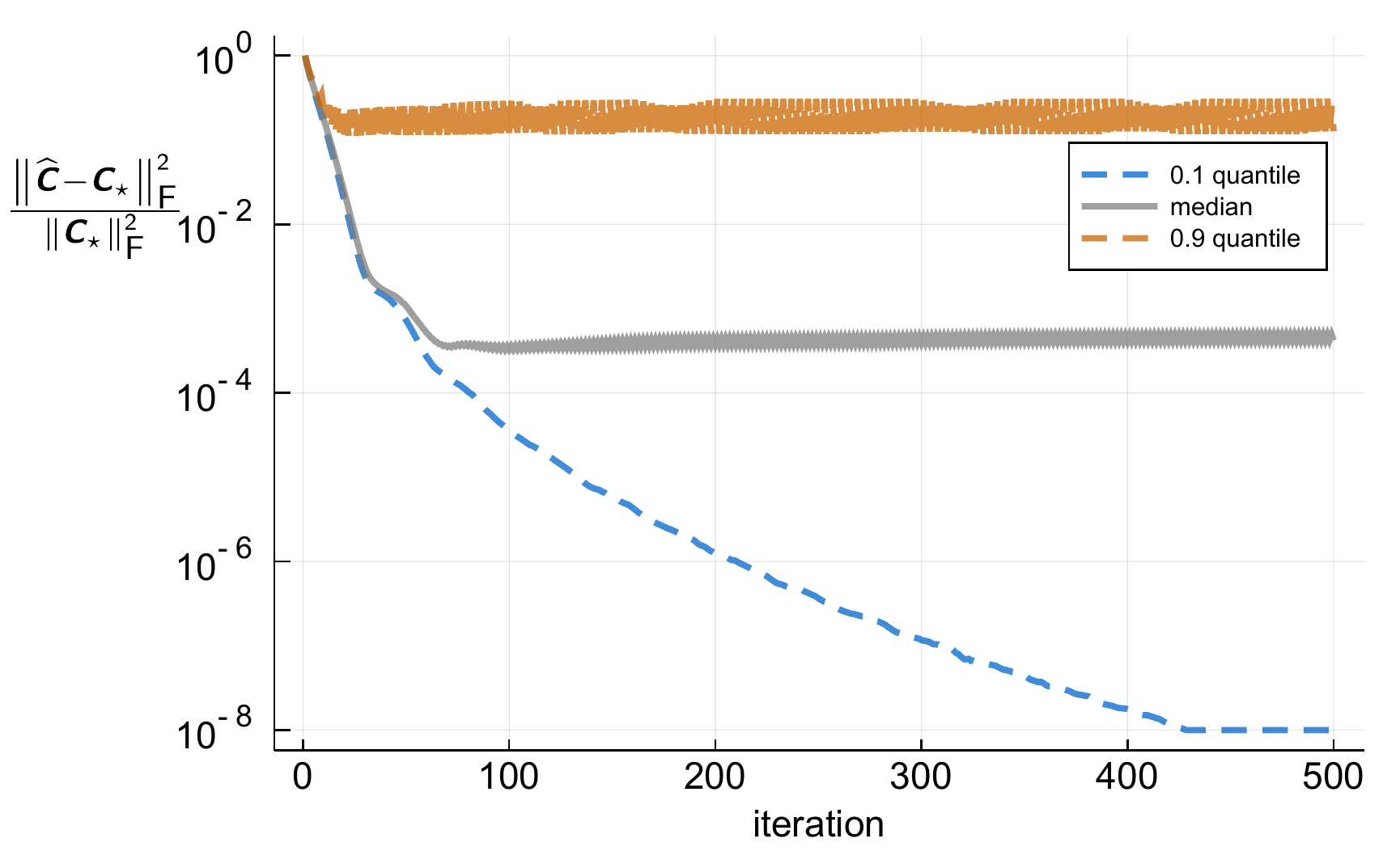}
            \caption{$\alpha = 0.8$, $\rho = 0.3$}
    \end{subfigure}

    \begin{subfigure}{0.49\textwidth}
            \centering
            \includegraphics[width=\textwidth]{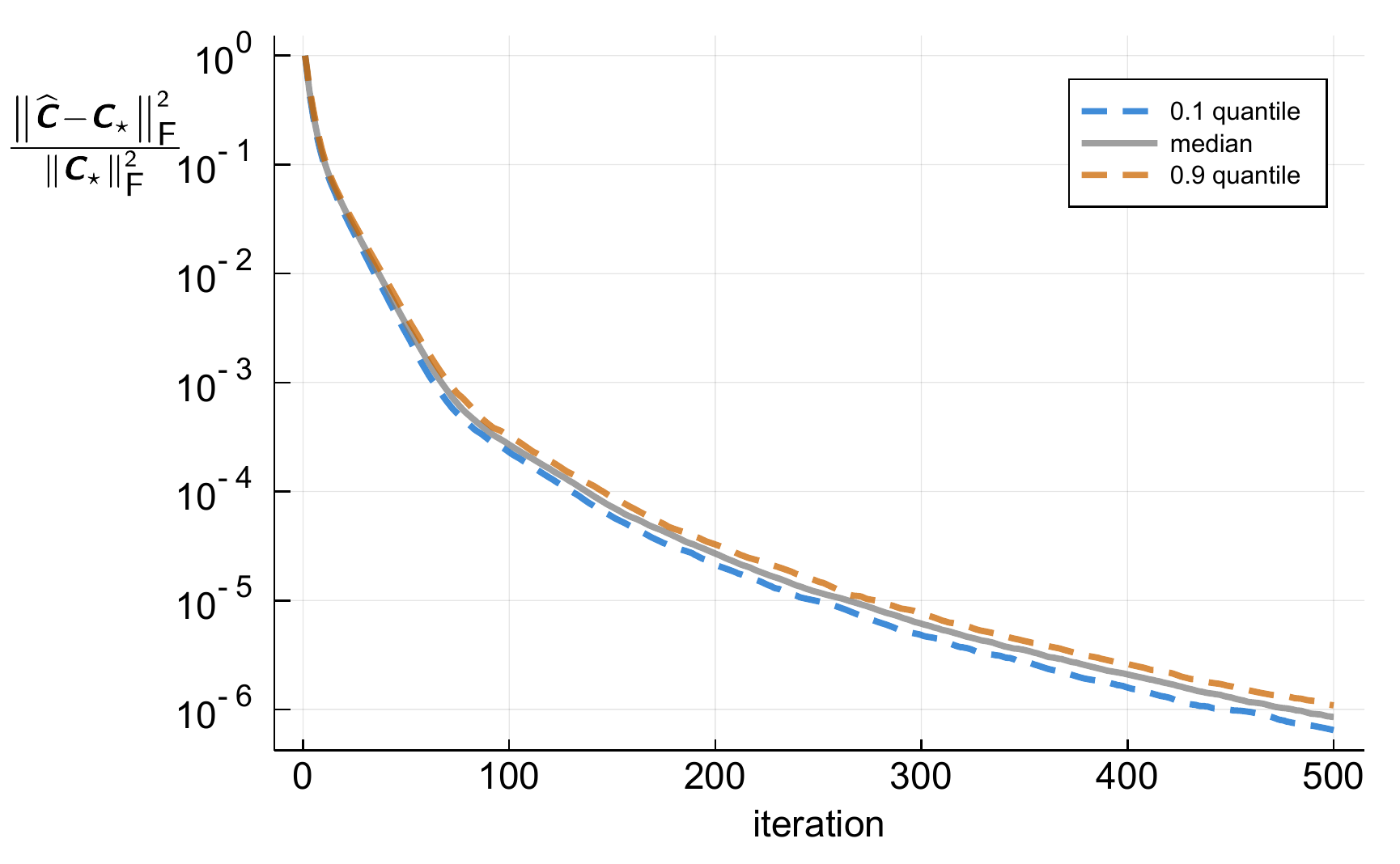}
            \caption{$\alpha = 0.2$, $\rho = 0$}
    \end{subfigure}
    \begin{subfigure}{0.49\textwidth}
            \centering
            \includegraphics[width=\textwidth]{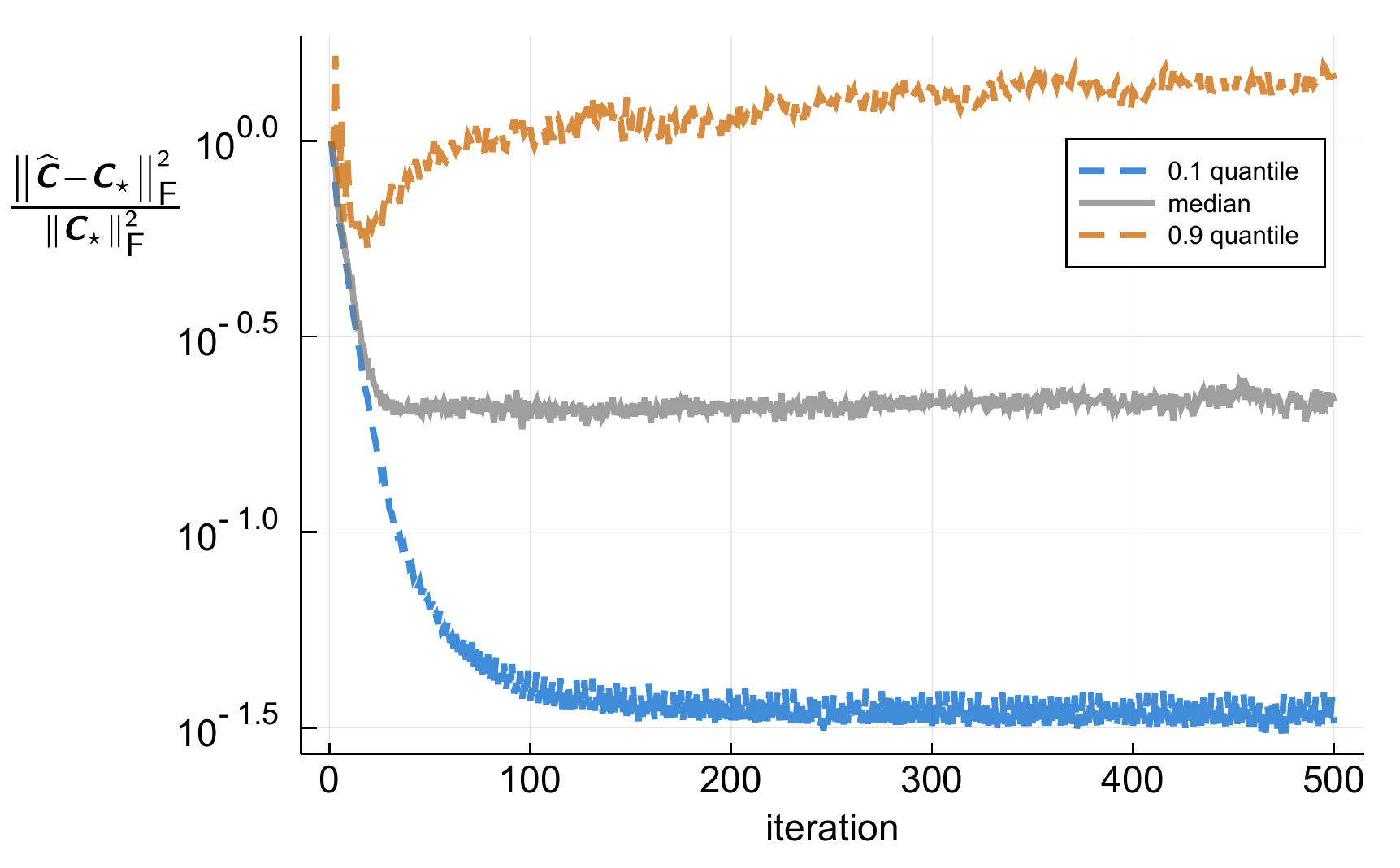}
            \caption{$\alpha = 0.8$, $\rho = 0$}
    \end{subfigure}
    \caption{Heavy-tailed model}
    \label{fig:heavy-tailed-model}
\end{figure}

Figures \ref{fig:Gaussian-model} and \ref{fig:heavy-tailed-model} depict the achieved relative error  under the Gaussian model and the heavy-tailed model for the chosen values of $\alpha$ and $\rho$, respectively. The solid lines show the median of the achieved relative error, whereas the dashed lines show the $0.1$ and $0.9$ quantiles of the relative error. Perhaps, the result that might strike as counter intuitive at first, is that the estimation performance is not monotonic with respect to the strength of stability. For instance, the plots in the first two rows of Figure \ref{fig:Gaussian-model} suggest that convergence is faster for the less stable system (i.e., $\alpha = 0.8$). A similar conclusion can be made regarding the plots in the first row of Figure \ref{fig:heavy-tailed-model} corresponding to linear activation functions. However, it appears that this behavior is sensitive to the level of nonlinearity, particularly in the case of the heavy-tailed input distributions.

\section{Proof of the main result}
\begin{proof}[\textbf{Proof of Theorem \ref{thm:main}}]
Recall the definition of $\mb{\varSigma}$ in \eqref{eq:covariance}. We would like to find a lower bound for the
smallest eigenvalue of $\mb{\varSigma}$ that holds with high probability.
Consider a sufficiently large integer $L$ as a \emph{stride} parameter
and for $\ell=0,\dotsc,L-1$ let
\begin{align*}
\mc T_{\ell} & =\left\{ t\,:\,L\le t< T\ \text{and}\ t=\ell\ \mr{mod}\ L\right\} \,,
\end{align*}
which partition $\left\{ L,\dotsc,T-1\right\} $ to sets of subsampled
time indices with stride $L$. For each $\ell=0,\dotsc,L-1$, we define
the ``restarted'' state variables $\mb x_{t}^{\left(\ell\right)}$
through the recursion
\begin{align*}
\mb x_{t+1}^{\left(\ell\right)} & =\begin{cases}
\mb 0 & ,\,t=\ell\ \mr{mod}\ L\\
\nabla f\left(\tg{\mb A}\mb x_{t}^{\left(\ell\right)}+\tg{\mb B}\mb u_{t}\right) & ,\,t\ne\ell\ \mr{mod}\ L\,,
\end{cases}
\end{align*}
and the corresponding restarted version of $\mb z_{t}$ as
\begin{align}
\mb z_{t}^{\left(\ell\right)} & =\bmx{\mb x_{t}^{\left(\ell\right)}\\
\beta\, \mb u_{t}
}\,. \label{eq:restarted-zt}
\end{align}
 For any $\mb w\in\mbb S^{n+p-1}$ we have
\begin{align*}
\mb w^{\T}\mb{\varSigma}\mb w & \ge\sum_{t=L}^{T-1}\left(\mb w^{\T}\mb z_{t}\right)^{2} =\sum_{\ell=0}^{L-1}\sum_{t\in\mc T_{\ell}}\left(\mb w^{\T}\mb z_{t}\right)^{2}\,.
\end{align*}
To find a lower bound for $\sum_{t\in\mc T_{\ell}}\left(\mb w^{\T}\mb z_{t}\right)^{2}$, the strategy
is to approximate this summation by its corresponding restarted version. Aggregating
the obtained bounds for all $\ell=0,\dotsc,L-1$ then yields the desired
lower bound for $\mb w^{\T}\mb{\varSigma}\mb w$.

By the Cauchy-Schwarz inequality we have
\begin{align*}
\left(\mb w^{\T}\mb z_{t}\right)^{2}+\left(\mb w^{\T}\left(\mb z_{t}-\mb z_{t}^{\left(\ell\right)}\right)\right)^{2} & \ge\frac{1}{2}\left(\mb w^{\T}\mb z_{t}^{\left(\ell\right)}\right)^{2}\,.
\end{align*}
Summing over $t\in\mc T_{\ell}$ and rearranging the terms then yields
\begin{align*}
\sum_{t\in\mc T_{\ell}}\left(\mb w^{\T}\mb z_{t}\right)^{2} & \ge\frac{1}{2}\underbrace{\sum_{t\in\mc T_{\ell}}\left(\mb w^{\T}\mb z_{t}^{\left(\ell\right)}\right)^{2}}_{\defeq S_{\ell}\left(\mb w\right)}-\underbrace{\sum_{t\in\mc T_{\ell}}\left(\mb w^{\T}\left(\mb z_{t}-\mb z_{t}^{\left(\ell\right)}\right)\right)^{2}}_{\defeq\tilde{S}_{\ell}\left(\mb w\right)}\,.
\end{align*}
Observe that the term $S_{\ell}\left(\mb w\right)$ is a sum of independent
random quadratic functions. Therefore, deriving a uniform lower bound
for $S_{\ell}\left(\mb w\right)$ is amenable to standard techniques.
We also need to establish a uniform upper bound for the term $\tilde{S}_{\ell}\left(\mb w\right)$
for which we leverage the contraction assumption.

Denote the matrix $\mb{M}$ with its $i$th column replaced by the zero vector as ${\mb M}_{\backslash i}$. The following lemma, whose proof is relegated to the appendix, provides a uniform lower bound on $\sum_{\ell=0}^{L-1}S_\ell(\mb w)$. The proof for this lemma is also provided in the appendix.

\begin{lem}[uniform lower bound for $S_{\ell}\left(\protect\mb w\right)$] \label{lem:Sl-lower-bound}
    With probability $\ge 1-\delta$, for all $\mb w\in\mbb S^{n+p-1}$ we have
 \begin{align*}
     \sum_{\ell=0}^{L-1} S_{\ell}(\mb w) & \ge \theta^2 L|\mc T_\ell|\left(\frac{0.1}{\max\{\eta,3\}}-\sqrt{\frac{2(n+p)\log\frac{eT/L}{(n+p)}+\log\frac{4L}{\delta}}{|\mc T_\ell|}}\right)\,,
 \end{align*}
 where $\theta$ is defined as in \eqref{eq:theta*}.
\end{lem}

Furthermore, we have the following lemma that establishes a uniform upperbound for $\sum_{\ell = 0}^{L-1}\tilde{S}_\ell(\mb w)$.

\begin{lem}[uniform upper bound for $\tilde{S}_{\ell}\left(\protect\mb w\right)$]
\label{lem:curlySl upper bound}
    Suppose that  $\mu \defeq p^{1/2}\norm{\mb B_\star}_{1\to 2}/\norm{\mb B_\star}_\F=O(1)$ and let $\epsilon > 0$ be a parameter. If for a certain absolute constant $c>0$, we have
    \begin{align}
            L  \ge 1+\frac{\log\left(\frac{c^2T}{\epsilon}\log\left(\frac{2(T-1)(p+1)}{\delta}\right)\left(\frac{\varLambda\norm{\tg{\mb B}}_{\F}}{1-\varLambda\norm{\tg{\mb A}}}\right)^2\right)}{\log\frac{1}{\varLambda\norm{\tg{\mb A}}}}\,, \label{eq:L}
    \end{align}
    then with probability $\ge 1-\delta$, we can guarantee
    \[
       \sum_{\ell=0}^{L-1}\tilde{S}_\ell(\mb{w}) \le \epsilon\,.
    \]
\end{lem}

Consquently, under \eqref{eq:L-main} and \eqref{eq:T-main}, it follows from Lemmas \ref{lem:Sl-lower-bound} and \ref{lem:curlySl upper bound} that
\begin{align*}
    \mb{w}^\T \mb{\varSigma}\mb{w} &\gtrsim \frac{\theta^2}{\max\{\eta,3\}} T\,.
\end{align*}
holds uniformly for all $\mb w\in \mbb S^{n+p-1}$ with probability $\ge 1- \delta$.
\end{proof}

\section*{Acknowledgements}
This work was supported in part by the Semiconductor Research Corporation (SRC) and DARPA.

{\small
\bibliography{references}
}

\appendix
\section{Proofs for technical lemmas}

\begin{proof}[\textbf{Proof of Lemma \ref{lem:Sl-lower-bound}}]
For each $\ell=0,\dotsc,L-1$, the vectors $\mb z_{t}^{\left(\ell\right)}$
with $t\in\mc T_{\ell}$ are independent and identically distributed.
Let $\theta>0$ be a parameter to
be specified later. Using a simple truncation we can write

\begin{align*}
\sum_{t\in\mc T_{\ell}}\left(\mb w^{\T}\mb z_{t}^{\left(\ell\right)}\right)^{2} & \ge \theta^2\sum_{t\in\mc T_{\ell}}\bbone\left(\left|\mb w^{\T}\mb z_{t}^{\left(\ell\right)}\right|\ge\theta\right)\,.
\end{align*}
To bound the right-hands side of the inequality above uniformly with respect to the set of binary functions
\begin{align*}
\mc F_{\ell} & \defeq\left\{ \mb z\mapsto\bbone\left(\left|\mb w^{\T}\mb z\right|\ge\theta\right)\,:\,\mb w\in\mbb S^{n+p-1}\right\} \,,
\end{align*}
we can resort to classic VC bounds (\citealp{Vapnik1971-Uniform};  \citealp[see also][chapters 13 \& 14]{Devroye2013-Probabilistic}). Particularly, because the VC dimension of $\mc{F}_\ell$ is no more than $2\left(n+p\right)$, with probability $\ge 1-\delta/L$ we have
\begin{align*}
 \frac{1}{\left|\mc T_{\ell}\right|}\sum_{t\in\mc T_{\ell}}\bbone\left(\left|\mb w^{\T}\mb z_{t}^{\left(\ell\right)}\right|\ge\theta\right) & \ge\P\left(\left|\mb w^{\T}\mb z_{t}^{\left(\ell\right)}\right|\ge\theta\right)-\sqrt{\frac{2(n+p)\log\frac{e\left|\mc T_{\ell}\right|}{n+p}+\log\frac{4L}{\delta}}{\left|\mc T_{\ell}\right|}}\,,
\end{align*}
 for all $\mb w\in\mbb S^{n+p-1}$.
 It only remains to find appropriate lower bounds for the probability in the summation. Lemma \ref{lem:lower-bound-P} below provides the needed lower bound.

 Taking the union bound over $\ell$ then shows that with probability $\ge 1-\delta$ we obtain
 \begin{align*}
    \sum_{\ell=0}^{L-1}\frac{1}{\left|\mc T_{\ell}\right|}\sum_{t\in\mc T_{\ell}}\bbone\left(\left|\mb w^{\T}\mb z_{t}^{\left(\ell\right)}\right|\ge\theta\right) & \ge L\left(\frac{0.1}{\max\{\eta,3\}}-\sqrt{\frac{2(n+p)\log\frac{e\left|\mc T_{\ell}\right|}{n+p}+\log\frac{4L}{\delta}}{\left|\mc T_{\ell}\right|}}\right)\,,
 \end{align*}
 which yields the desired bound.
 \end{proof}

\begin{proof}[\textbf{Proof of Lemma \ref{lem:curlySl upper bound}}]
Recall the definition of $\mb{z}_t^{(\ell)}$ in \eqref{eq:restarted-zt}. For every $t\in\mc T_{\ell}$ and $\mb w\in\mbb S^{n+p-1}$ we have
\begin{align*}
\left(\mb w^{\T}\left(\mb z_{t}-\mb z_{t}^{\left(\ell\right)}\right)\right)^{2} & \le\norm{\mb z_{t}-\mb z_{t}^{\left(\ell\right)}}_{2}^{2}\\
 & =\norm{\nabla f\left(\tg{\mb C}\mb z_{t-1}\right)-\nabla f\left(\tg{\mb C}\mb z_{t-1}^{\left(\ell\right)}\right)}_{2}^{2}\,.
\end{align*}
Furthermore, we can write
\begin{align*}
\norm{\nabla f\left(\tg{\mb C}\mb z_{t-1}\right)-\nabla f\left(\tg{\mb C}\mb z_{t-1}^{\left(\ell\right)}\right)}_{2}^{2} & \le\varLambda^{2}\norm{\tg{\mb C}\left(\mb z_{t-1}-\mb z_{t-1}^{\left(\ell\right)}\right)}_{2}^{2}\\
 & =\varLambda^{2}\norm{\tg{\mb A}\left(\nabla f\left(\tg{\mb C}\mb z_{t-2}\right)-\nabla f\left(\tg{\mb C}\mb z_{t-2}^{\left(\ell\right)}\right)\right)}_{2}^{2}\\
 & \le\left(\varLambda\norm{\tg{\mb A}}\right)^{2}\norm{\left(\nabla f\left(\tg{\mb C}\mb z_{t-2}\right)-\nabla f\left(\tg{\mb C}\mb z_{t-2}^{\left(\ell\right)}\right)\right)}_{2}^{2}\,.
\end{align*}
 Using the above inequality recursively yields
\begin{align*}
\norm{\nabla f\left(\tg{\mb C}\mb z_{t-1}\right)-\nabla f\left(\tg{\mb C}\mb z_{t-1}^{\left(\ell\right)}\right)}_{2}^{2} & \le\left(\varLambda\norm{\tg{\mb A}}\right)^{2\left(L-2\right)}\norm{\nabla f\left(\tg{\mb C}\mb z_{t-L+1}\right)-\nabla f\left(\tg{\mb C}\mb z_{t-L+1}^{\left(\ell\right)}\right)}_{2}^{2}\\
 & \le\left(\varLambda\norm{\tg{\mb A}}\right)^{2\left(L-1\right)}\norm{\mb x_{t-L}}_{2}^{2}\,.
\end{align*}
 Therefore, we deduce that
\begin{align}
\left(\mb w^{\T}\left(\mb z_{t}-\mb z_{t}^{\left(\ell\right)}\right)\right)^{2} & \le\left(\varLambda\norm{\tg{\mb A}}\right)^{2\left(L-1\right)}\norm{\mb x_{t-L}}_{2}^{2}\,.\label{eq:deviation-t}
\end{align}
 Furthermore, for any time index $s\ge1$ we have
\begin{align*}
\norm{\mb x_{s}}_{2} & \le\varLambda\norm{\tg{\mb A}\mb x_{s-1}+\tg{\mb B}\mb u_{s-1}}_{2}\\
 & \le\varLambda\norm{\tg{\mb A}}\norm{\mb x_{s-1}}_{2}+\varLambda\norm{\tg{\mb B}\mb u_{s-1}}_{2}\,.
\end{align*}
 Therefore, we can write
\begin{align*}
\max_{1\le s\le T-1}\norm{\mb x_{s}}_{2} & \le\varLambda\norm{\tg{\mb A}}\max_{1\le s\le T-1}\norm{\mb x_{s-1}}_{2}+\varLambda\max_{1\le s\le T-1}\norm{\tg{\mb B}\mb u_{s-1}}_{2}\,,
\end{align*}
which implies
\begin{align*}
\max_{1\le s\le T-1}\norm{\mb x_{s}}_{2} & \le\frac{\varLambda}{1-\varLambda\norm{\tg{\mb A}}}\max_{1\le s\le T-1}\norm{\tg{\mb B}\mb u_{s-1}}_{2}\,.
\end{align*}

 Since $\mu = p^{1/2}\norm{\mb B_\star}_{1\to 2}/\norm{\mb B_\star}_\F=O(1)$ by assumption, using the matrix Bernstein inequality, stated in Lemma \ref{lem:Bu-Bernstein} below, for each $s=1,\dotsc,T-1$, with probability $\ge 1- \delta/(T-1)$ we have
  \begin{align*}
     \norm{\mb B_\star \mb u_{s-1}}_2 & \le c\norm{\mb B_\star}_\F\log^{\frac{1}{2}}\left(\frac{2(T-1)(p+1)}{\delta}\right)\,,
 \end{align*}
 for some absolute constant $c>0$. It then follows from a simple union bound that
 \begin{align*}
     \max_{1\le s\le T-1}\norm{\mb B_\star \mb u_{s-1}}_2 &\le c\norm{\mb B_\star}_\F\log^{\frac{1}{2}}\left(\frac{2(T-1)(p+1)}{\delta}\right)\,,
 \end{align*}
 holds with probability $\ge 1-\delta$. Consequently,
\begin{align*}
\max_{1\le s\le T-1}\norm{\mb x_{s}}_{2} & \le c\log^{\frac{1}{2}}\left(\frac{2(T-1)(p+1)}{\delta}\right)\frac{\varLambda\norm{\tg{\mb B}}_{\F}}{1-\varLambda\norm{\tg{\mb A}}}\,,
\end{align*}
holds with probability $\ge1-\delta$. Under the same event and in
view of \eqref{eq:deviation-t} we have
\begin{align*}
\left(\mb w^{\T}\left(\mb z_{t}-\mb z_{t}^{\left(\ell\right)}\right)\right)^{2} & \le\left(\varLambda\norm{\tg{\mb A}}\right)^{2\left(L-1\right)}c^2\log\left(\frac{2(T-1)(p+1)}{\delta}\right)\left(\frac{\varLambda\norm{\tg{\mb B}}_{\F}}{1-\varLambda\norm{\tg{\mb A}}}\right)^2\,,
\end{align*}
 for all $\mb w\in\mbb S^{n+p-1}$, $0\le\ell\le L-1$, and $t\in\mc T_{\ell}$.
Summation over $t\in\mc T_{\ell}$ then yields
\begin{align*}
\tilde{S}_{\ell}\left(\mb w\right) & =\sum_{t\in\mc T_{\ell}}\left(\mb w^{\T}\left(\mb z_{t}-\mb z_{t}^{\left(\ell\right)}\right)\right)^{2} \\
& \le\frac{T}{L}\left(\varLambda\norm{\tg{\mb A}}\right)^{2\left(L-1\right)}c^2\log\left(\frac{2(T-1)(p+1)}{\delta}\right)\left(\frac{\varLambda\norm{\tg{\mb B}}_{\F}}{1-\varLambda\norm{\tg{\mb A}}}\right)^2\,.
\end{align*}
 Therefore, for $\epsilon>0$ if
\begin{align*}
L & \ge 1+\frac{\log\left(\frac{c^2T}{\epsilon}\log\left(\frac{2(T-1)(p+1)}{\delta}\right)\left(\frac{\varLambda\norm{\tg{\mb B}}_{\F}}{1-\varLambda\norm{\tg{\mb A}}}\right)^2\right)}{\log\frac{1}{\varLambda\norm{\tg{\mb A}}}}\,,
\end{align*}
then with probability $\ge1-\delta$ for all $\mb w\in\mbb S^{n+p-1}$
we have
\begin{align*}
\sum_{\ell=0}^{L-1}\tilde{S}_{\ell}\left(\mb w\right) & \le\epsilon\,.
\end{align*}
\end{proof}

\section{Auxiliary lemmas}
We use a special case of a matrix Bernstein inequality \citep[Corollary 2.1]{Koltchinskii2011-Oracle}. For reference, the following lemma states the special inequality we need; we omit the proof and refer the reader to  \citep{Koltchinskii2011-Oracle} for the general Bernstein inequality.

\begin{lem}\label{lem:Bu-Bernstein}
    Suppose that $\mb{u}$ obeys the Assumption \ref{asm:regularity-u}. Furthermore, define a coherence parameter for $\tg{\mb B}$ as $\mu \defeq p^{1/2}\norm{\tg{\mb B}}_{1\to 2}/\norm{\tg{\mb B}}_\F$.  Then, for some absolute constant $c>0$, and any $\gamma \in (0,1]$, the bound
     \begin{align*}
        & \norm{\tg{\mb{B}}\mb{u}}_2 \\
        & \le \max\left\lbrace c^{\frac{1}{2}}\log^{\frac{1}{2}}\left(2\gamma^{-1}(p+1)\right),\ c\max\{K,2\}\mu\,\log^{\frac{1}{\alpha}}\left(\max\{K,2\}\mu\right)\, \frac{\log\left(2\gamma^{-1}(p+1)\right)}{p^{1/2}}\right\rbrace \norm{\tg{\mb{B}}}_{\F}\,,
     \end{align*}
holds with probability $\ge 1-\gamma$.
In particular, if $\mu=O(1)$, meaning that the weight of $\tg{\mb B}$ is distributed almost evenly across its columns, and $p$ is sufficiently large, the bound stated above effectively reduces to
\begin{align*}
    \norm{\tg{\mb B}\mb{u}}_2 \le c\norm{\tg{\mb B}}_\F \log^{\frac{1}{2}}\left(2\gamma^{-1}(p+1)\right)\,,
\end{align*}
for some absolute constant $c>0$.
\end{lem}

In general, the coherence parameter $\mu$ defined in Lemma \ref{lem:Bu-Bernstein} obeys $1\le \mu\le p^{1/2}$. However, we assume we operate in the scenario that $\mu =O(1)$ so that we apply the simpler bound stated in the lemma. Therefore, choosing $\gamma = 1/p$ and for a sufficiently large $p$ we have
 \begin{align*}
    &\hphantom{\ge} \P\left(\left|\mb{w}^\T   \bmx{
                                                    F(\tg{\mb{A}}\mb{x}^{(\ell)}_{t-1})\tg{\mb{B}}\mb{u}_{t-1}\\
                                                    \beta\,\mb{u}_t
                                                }
                            \right|- \varepsilon \norm{\tg{\mb{B}}\mb{u}_{t-1}}_2\ge \theta\right)\\
                            & \ge \P\left(\left|\mb{w}^\T   \bmx{
                                                    F(\tg{\mb{A}}\mb{x}^{(\ell)}_{t-1})\tg{\mb{B}}\mb{u}_{t-1}\\
                                                    \beta\,\mb{u}_t
                                                }
                            \right| \ge \theta + c\varepsilon \log^{\frac{1}{2}}\left(2p(p+1)\right)\norm{\tg{\mb{B}}}_\F\right) - \frac{1}{p}\,.
 \end{align*}
 for some absolute constant $c>0$.

 \begin{lem}[lower bound for the probabilities]\label{lem:lower-bound-P}
   With $\theta$ defined as in \eqref{eq:theta*},
 for each $\ell\in \{0,1,\dotsc,L-1\}$, and every $t\in\mc{T}_\ell$ we have
        \[\P\left(\left|\mb{w}^\T\mb{z}_t^{(\ell)}\right| \ge  \theta \right) \ge \frac{0.1}{\max\{\eta,3\}}\]
 \end{lem}

 \begin{proof}
 For $t=0,1,\dotsc$, let $i_t$ be i.i.d. integers uniformly distributed over $\{1,\dotsc,p\}$, independent of everything else. For any vector $\mb{v}$, we use the notation $\mb{v}^{-i}$ to denote the vector obtained by flipping the sign of the $i$th coordinate of $\mb{v}$. Furthermore, for $t\in \mc{T}_\ell$ let
 \[
 \overline{\mb{z}}^{(\ell)}_t = \bmx{
                                    \nabla f(\tg{\mb{A}}\mb{x}^{(\ell)}_{t-1}+\tg{\mb{B}}\mb{u}^{-i_{t-1}}_{t-1}\\[1ex]
                                    -\beta\,\mb{u}_t
                                }\,.
 \]
  Recall that, by assumption, $\mb{u}_{t-1}$ and $\mb{u}_t$ have coordinates with independent symmetric distributions. Therefore, it is straightforward to show that $\mb{z}^{(\ell)}_t$ and $\overline{\mb{z}}^{(\ell)}_t$ are identically distributed, and for any $\theta >0$ we can write
 \begin{align*}
     \P\left(\left|\mb w^{\T}\mb z_{t}^{\left(\ell\right)}\right|\ge\theta\right) &=\frac{1}{2}\P\left(\left|\mb w^{\T}\mb z_{t}^{\left(\ell\right)}\right|\ge\theta\right)+\frac{1}{2}\P\left(\left|\mb w^{\T}\overline{\mb{z}}_{t}^{\left(\ell\right)}\right|\ge\theta\right)\\
     & \ge \frac{1}{2}\P\left(\left|\mb w^{\T}\mb z_{t}^{\left(\ell\right)}\right|+\left|\mb w^{\T}\overline{\mb{z}}_{t}^{\left(\ell\right)}\right|\ge2\theta\right)\,.
     \end{align*}
     Then, it follows from the triangle inequality, and the assumption \eqref{eq:local-approx-linearity},  that
     \begin{align}
     \P\left(\left|\mb w^{\T}\mb z_{t}^{\left(\ell\right)}\right|\ge\theta\right)  & \ge \frac{1}{2}\P\left(\left|\mb w^{\T}\left(\mb z_{t}^{\left(\ell\right)}-\overline{\mb{z}}_{t}^{\left(\ell\right)}\right)\right|\ge2\theta\right)\nonumber\\
     & \ge \frac{1}{2}\P\left(\left|\mb{w}^\T   \bmx{
                                                    \nabla f(\tg{\mb{A}}\mb{x}^{(\ell)}_{t-1} + \tg{\mb{B}}\mb{u}_{t-1}) - \nabla f(\tg{\mb{A}}\mb{x}^{(\ell)}_{t-1} + \tg{\mb{B}}\mb{u}_{t-1}^{-i_{t-1}})\\
                                                    2\beta\,\mb{u}_t
                                                }
                            \right|\ge 2\theta\right)\nonumber\\
    & \ge \frac{1}{2}\P\left(\left|\mb{w}^\T   \bmx{
                                                    F(\tg{\mb{A}}\mb{x}^{(\ell)}_{t-1})\tg{\mb{B}}\left(\frac{1}{2}\mb{u}_{t-1}+\frac{1}{2}\mb{u}^{-i_{t-1}}_{t-1}\right)\\
                                                    \beta\,\mb{u}_t
                                                }
                            \right|- \varepsilon\norm{\tg{\mb{B}}\left(\frac{1}{2}\mb{u}_{t-1}-\frac{1}{2}\mb{u}^{-i_{t-1}}_{t-1}\right)}_2\ge \theta\right)\,.\label{eq:lem-LBP0}
 \end{align}
 Furthermore, for any $\gamma \in (0,1]$ we can write
 \begin{align}
&\hphantom{\ge} \P\left(\left|\mb{w}^\T   \bmx{                                                                                           F(\tg{\mb{A}}\mb{x}^{(\ell)}_{t-1})\tg{\mb{B}}\left(\frac{1}{2}\mb{u}_{t-1}+\frac{1}{2}\mb{u}^{-i_{t-1}}_{t-1}\right)\\
                                                    \beta\,\mb{u}_t
                                                }
                            \right|- \varepsilon\norm{\tg{\mb{B}}\left(\frac{1}{2}\mb{u}_{t-1}-\frac{1}{2}\mb{u}^{-i_{t-1}}_{t-1}\right)}_2\ge \theta\right)\nonumber \\
                            & + \P\left(\norm{\tg{\mb{B}}\left(\frac{1}{2}\mb{u}_{t-1}-\frac{1}{2}\mb{u}^{-i_{t-1}}_{t-1}\right)}_2 \ge K\log^{\frac{1}{\alpha}}\left(\frac{2}{\gamma}\right)\norm{\tg{\mb{B}}}_{1\to 2}\right)\nonumber\\
                            & \ge \P\left(\left|\mb{w}^\T   \bmx{
                                                    F(\tg{\mb{A}}\mb{x}^{(\ell)}_{t-1})\tg{\mb{B}}\left(\frac{1}{2}\mb{u}_{t-1}+\frac{1}{2}\mb{u}^{-i_{t-1}}_{t-1}\right)\\
                                                    \beta\,\mb{u}_t
                                                }
                            \right| \ge \theta + \varepsilon K\log^{\frac{1}{\alpha}}\left(\frac{2}{\gamma}\right)\norm{\tg{\mb{B}}}_{1\to 2}\right)\,. \label{eq:lem-LBP1}
 \end{align}
 Observe that $(\mb{v}-\mb{v}^{-i})/2=\mb{v}|_i$ and $(\mb{v}+\mb{v}^{-i})/2=\mb{v}|_{\backslash i}$ are respectively the selectors of the $i$th coordinate and its complement. With this convention, on one hand we can write
 \begin{align}
     & \P\left(\norm{\tg{\mb{B}}\left(\frac{1}{2}\mb{u}_{t-1}-\frac{1}{2}\mb{u}^{-i_{t-1}}_{t-1}\right)}_2 \ge K\log^{\frac{1}{\alpha}}\left(\frac{2}{\gamma}\right)\norm{\tg{\mb{B}}}_{1\to 2}\right)\nonumber\\
     & = \P\left(\norm{\tg{\mb B}\left({\mb{u}_{t-1}\mid}_{i_{t-1}}\right)}_2 \ge K\log^{\frac{1}{\alpha}}\left(\frac{2}{\gamma}\right)\norm{\tg{\mb{B}}}_{1\to 2}\right)\nonumber\\
     & \le \gamma\,, \label{eq:lem-LBP2}
 \end{align}
 where the third line follows from the fact that $\norm{\tg{\mb B}}_{1\to 2}$ is equal to the greatest $\ell_2$ norm of the columns of $\tg{\mb B}$, and that under the assumption \eqref{eq:Orlicz-psi-alpha} we have
        \[\P\left(\left|\left(\mb{u}_{t-1}\right)_{i_{t-1}}\right|\ge K\log^{\frac{1}{\alpha}}\left(\frac{2}{\gamma}\right)\right) \le \gamma\,.\]
On the other hand, we can write \[\tg{\mb{B}}\left(\frac{1}{2}\mb{u}_{t-1}+\frac{1}{2}\mb{u}^{-i_{t-1}}_{t-1}\right) = {\tg{\mb B}}_{\backslash i_{t-1}}\mb{u}_{t-1}\,\]
and invoke Lemma \ref{lem:Apply-PZ} below to obtain


\begin{equation}
     \begin{aligned}
          & \hphantom{\ge} \P\left(\left|\mb{w}^\T   \bmx{
                                                    F(\tg{\mb{A}}\mb{x}^{(\ell)}_{t-1}){\tg{\mb{B}}}_{\backslash i_{t-1}}\mb{u}_{t-1}\\
                                                    \beta\,\mb{u}_t
                                                }
                            \right| \ge 0.36\min_{i=1,\dotsc,p}\min\left\{\beta, (\lambda -\varepsilon) \lambda^{1/2}_{\min}\left({\tg{\mb B}}_{\backslash i}{\tg{\mb B}}^\T_{\backslash i}\right)\right\}\right)\\
            & \ge \frac{0.4}{\max\{\eta, 3\}}
     \end{aligned} \label{eq:lem-LBP3}
     \end{equation}

Therefore, recalling the assumed condition \eqref{eq:local-approx-linearity}, by choosing
\[    \theta = \theta_{\alpha,\beta,\varepsilon,\lambda,K,\tg{\mb B}}\,,\]
and
\[ \gamma = \frac{0.2}{\max\{\eta, 3\}}\,,\]
and in view of \eqref{eq:lem-LBP0}, \eqref{eq:lem-LBP1}, \eqref{eq:lem-LBP2}, and \eqref{eq:lem-LBP3} we obtain the desired bound
\[\P\left(\left|\mb{w}^\T\mb{z}_t^{(\ell)}\right| \ge \theta_{\alpha,\beta,\varepsilon,\lambda,K,\tg{\mb B}} \right) \ge \frac{0.1}{\max\{\eta,3\}}\,.\]
\end{proof}

  \begin{lem}\label{lem:Apply-PZ}
 With the notation and conditions as in Lemma \ref{lem:lower-bound-P} we have
     \begin{equation*}
     \begin{aligned}
          & \hphantom{\ge} \P\left(\left|\mb{w}^\T   \bmx{
                                                    F(\tg{\mb{A}}\mb{x}^{(\ell)}_{t-1}){\tg{\mb{B}}}_{\backslash i_{t-1}}\mb{u}_{t-1}\\
                                                    \beta\,\mb{u}_t
                                                }
                            \right| \ge 0.36\min_{i=1,\dotsc,p}\min\left\{\beta, (\lambda -\varepsilon) \lambda^{1/2}_{\min}\left({\tg{\mb B}}_{\backslash i}{\tg{\mb B}}^\T_{\backslash i}\right)\right\}\right)\\
            & \ge \frac{0.4}{\max\{\eta, 3\}}
     \end{aligned}
     \end{equation*}
 \end{lem}
 \begin{proof}
By conditioning on $\mb{x}^{(\ell)}_{t-1}$ and applying the Paley-Zygmund inequality (\citealp{PaleyZygmund1932-Analytic}; \citealp[][Corollary 3.3.2]{delaPenaGine1999-Decoupling}) we have
\begin{align}
    & \hphantom{\ge}\P\left(\left|\mb{w}^\T   \bmx{
                                                    F(\tg{\mb{A}}\mb{x}^{(\ell)}_{t-1}){\tg{\mb{B}}}_{\backslash i_{t-1}}\mb{u}_{t-1}\\
                                                    \beta\,\mb{u}_t
                                                }
                            \right|^2 \ge 0.36\,\E\left(\left|\mb{w}^\T   \bmx{
                                                    F(\tg{\mb{A}}\mb{x}^{(\ell)}_{t-1}){\tg{\mb{B}}}_{\backslash i_{t-1}}\mb{u}_{t-1}\\
                                                    \beta\,\mb{u}_t
                                                }
                            \right|^2\st[\Big|]\mb{x}^{(\ell)}_{t-1}\right) \st[\Big|] \mb{x}^{(\ell)}_{t-1}\right) \nonumber\\
    & \ge 0.4\, \frac{\left(\E\left(\left|\mb{w}^\T   \bmx{
                                                    F(\tg{\mb{A}}\mb{x}^{(\ell)}_{t-1}){\tg{\mb{B}}}_{\backslash i_{t-1}}\mb{u}_{t-1}\\
                                                    \beta\,\mb{u}_t
                                                }
                            \right|^2\st[\Big|]\mb{x}^{(\ell)}_{t-1}\right)\right)^2}{\E\left(\left|\mb{w}^\T   \bmx{
                                                    F(\tg{\mb{A}}\mb{x}^{(\ell)}_{t-1}){\tg{\mb{B}}}_{\backslash i_{t-1}}\mb{u}_{t-1}\\
                                                    \beta\,\mb{u}_t
                                                }
                            \right|^4\st[\Big|]\mb{x}^{(\ell)}_{t-1}\right)} \label{eq:PZ}
\end{align}
Using the assumption that $\mb{u}_{t-1}$ and $\mb{u}_t$ are independent, zero-mean, and isotropic we obtain
\begin{align*}
    \E\left(\left|\mb{w}^\T   \bmx{
                                                    F(\tg{\mb{A}}\mb{x}^{(\ell)}_{t-1}){\tg{\mb{B}}}_{\backslash i_{t-1}}\mb{u}_{t-1}\\
                                                    \beta\,\mb{u}_t
                                                }
                            \right|^2\st[\Big|]\mb{x}^{(\ell)}_{t-1}\right) &= \norm{\bmx{
                                                                                    \left(F(\mb{x}^{(\ell)}_{t-1}){\tg{\mb{B}}}_{\backslash i_{t-1}}\right)^\T & \mb{0}\\
                                                                                    \mb{0} & \beta\,\mb{I}
                                                                              }
                                                                              \mb{w}}_2^2\,.
\end{align*}
Furthermore, in view of Lemma \ref{lem:concatenated-moments}, the denominator in \eqref{eq:PZ} can be bounded from above as
\begin{align*}
    \E\left(\left|\mb{w}^\T   \bmx{
                                                    F(\tg{\mb{A}}\mb{x}^{(\ell)}_{t-1}){\tg{\mb{B}}}_{\backslash i_{t-1}}\mb{u}_{t-1}\\
                                                    \beta\,\mb{u}_t
                                                }
                            \right|^4\st[\Big|]\mb{x}^{(\ell)}_{t-1}\right)  & \le \max\{\eta,3\} \norm{\bmx{
                                                                                    \left(F(\mb{x}^{(\ell)}_{t-1}){\tg{\mb{B}}}_{\backslash i_{t-1}}\right)^\T & \mb{0}\\
                                                                                    \mb{0} & \beta\,\mb{I}
                                                                              }
                                                                              \mb{w}}_2^4\,.\\
\end{align*}
Therefore, \eqref{eq:PZ} reduces to
\begin{align}
    & \P\left(\left|\mb{w}^\T   \bmx{
                                                    F(\tg{\mb{A}}\mb{x}^{(\ell)}_{t-1}){\tg{\mb{B}}}_{\backslash i_{t-1}}\mb{u}_{t-1}\\
                                                    \beta\,\mb{u}_t
                                                }
                            \right|^2 \ge 0.36\,\E\left(\left|\mb{w}^\T   \bmx{
                                                    F(\tg{\mb{A}}\mb{x}^{(\ell)}_{t-1}){\tg{\mb{B}}}_{\backslash i_{t-1}}\mb{u}_{t-1}\\
                                                    \beta\,\mb{u}_t
                                                }
                            \right|^2\st[\Big|]\mb{x}^{(\ell)}_{t-1}\right) \st[\Big|] \mb{x}^{(\ell)}_{t-1}\right)\nonumber \\
    & \ge \frac{0.4}{\max\{\eta, 3\}}\,.\label{eq:PZ1}
\end{align}
It follows from Lemma \ref{lem:bound-F} that
\[
    \lambda_{\min}\left(
        \bmx{
            F(\mb{y}){\tg{\mb{B}}}_{\backslash i} & \mb{0}\\
            \mb{0} & \beta\mb{I}
        }
        \bmx{
            \left(F(\mb{y}){\tg{\mb{B}}}_{\backslash i}\right)^\T & \mb{0}\\
            \mb{0} & \beta\mb{I}
        }
        \right) \ge \min\{\beta^2, (\lambda - \varepsilon)^2\lambda_{\min}\left({\tg{\mb{B}}}_{\backslash i}{\tg{\mb{B}}}_{\backslash i}^\T\right)\,\}
\]
for all $\mb{y}$. In particular,
\[
                           \min\{\beta^2,\left(\lambda - \varepsilon\right)^2\lambda^2_{\min}\left({\tg{\mb{B}}}_{\backslash i}{\tg{\mb{B}}}_{\backslash i}^\T\right)\}
                           \le \norm{\bmx{
                                        \left(F(\mb{x}^{(\ell)}_{t-1}){\tg{\mb{B}}}_{\backslash i}\right)^\T & \mb{0}\\
                                        \mb{0} & \beta \mb{I}
                                  }
                                  \mb{w}}_2^2\,.
\]
Therefore, the conditional expectation in \eqref{eq:PZ1} can be replaced by \[\min_{i=1,\dotsc,p}\min\{\beta^2,\left(\lambda - \varepsilon\right)^2\lambda_{\min}\left({\tg{\mb{B}}}_{\backslash i}{\tg{\mb{B}}}_{\backslash i}^\T\right)\}\,.\] Finally, taking the expectation with respect to $\mb{x}_{t-1}^{(\ell)}$ completes the proof.
\end{proof}

\end{document}